\documentclass[11pt,
              a4paper,
              index=totoc,
              headsepline,
              footsepline,
              BCOR=12mm,
              DIV=13]{scrbook}

\def\university{Technische Universit{\"a}t M{\"u}nchen}

\def\doctype{Master's Thesis}

\def\title{Modelling Compositionality and Structure Dependence in Natural Language}
\def\author{Karthikeya Ramesh Kaushik}

\def\Advisorone{Dr.\ Andrea E Martin}

\def\date{September 30, 2020}

\def\keywords{{keyword1}, {keyword2}, {keyword3}}

\def\metaTitle{\title}
\def\metaAuthor{\author}
\def\metaSubject{\doctype\ -\ \university}
\def\metaKeywords{\keywords}

\def\footertext{}

\usepackage{palatino}

\usepackage{scrlayer-scrpage}
\usepackage{scrhack}
\ifoot[\footertext]{\footertext} 

\usepackage{ifthen}

\usepackage{rotating}

\usepackage{hyperref}

\usepackage[toc, xindy]{glossaries}

\usepackage{makeidx}

\usepackage{graphicx}
\graphicspath{{figures/}}

\usepackage{float}

\usepackage{amsmath}

\usepackage{amsfonts}

\usepackage{tabu}

\usepackage[dvipsnames]{xcolor}

\definecolor{Pantone300C}{HTML}{0065BD} 
\definecolor{Pantone301}{HTML}{005293}  
\definecolor{Pantone540}{HTML}{003359}  
\definecolor{DarkGray}{HTML}{333333}    
\definecolor{MediumGray}{HTML}{808080}  
\definecolor{LightGray}{HTML}{CCCCC6}   
\definecolor{Pantone7527}{HTML}{DAD7CB} 
\definecolor{Pantone158}{HTML}{E37222}  
\definecolor{Pantone383}{HTML}{A2AD00}  
\definecolor{Pantone283}{HTML}{98C6EA}  
\definecolor{Pantone542}{HTML}{64A0C8}  

\def\colorLinks{Pantone300C}
\def\colorUrl{Pantone542}
\def\colorCitations{Pantone158}

\hypersetup{
  linkcolor=\colorLinks,%
  urlcolor=\colorUrl,%
  citecolor=\colorCitations
}

\usepackage[ocgcolorlinks]{ocgx2}[2017/03/30]

\hypersetup{
  pdftitle={\metaTitle},%
  pdfauthor={\metaAuthor},%
  pdfkeywords={\metaKeywords},%
  pdfsubject={\metaSubject}
}

\usepackage{hyperxmp}

\usepackage{xspace}
\usepackage{enumerate}
\usepackage[all,cmtip]{xy}
\usepackage{float}
\usepackage{amsfonts}
\usepackage[utf8]{inputenc}
\usepackage[english]{babel}
\usepackage{amsmath}
\usepackage{graphicx}
\usepackage[colorinlistoftodos]{todonotes}
\usepackage{hyperref}
\usepackage{qtree}
\usepackage{tikz}
\usetikzlibrary{matrix}
\usepackage{tikz-cd}
\usepackage{amsthm}
\usepackage{enumitem}
\usepackage{array}
\usepackage{subfigure}
\usepackage{lipsum}
\usepackage{rotating}
\usepackage{tabularx}
\usepackage{longtable}

\newtheorem{theorem}{Theorem}[chapter]
\newtheorem{lemma}[theorem]{Lemma}

\newtheorem{definition}[theorem]{Definition}

\newtheorem{example}[theorem]{Example}

\newcommand{\clearemptydoublepage}{%
  \ifthenelse{\boolean{@twoside}}{\newpage{\pagestyle{empty}\cleardoublepage}}%
  {\clearpage}}

\newglossaryentry{computer}
{
  name=computer,
  description={is a programmable machine that receives input,
               stores and manipulates data, and provides
               output in a useful format}
}

\makeglossaries

\begin{document}

 \frontmatter

 \def\bcorcor{0.15cm}
\addtolength{\hoffset}{\bcorcor}

\thispagestyle{empty}

\vspace{12cm}
\begin{center}
	{\huge \textbf \title}\\
	\vspace{15mm}
	{\LARGE  \author}\\
	\vspace{10mm}
	{\Large Advised by \Advisorone}\\
	\vspace{\fill}
\end{center}






\clearemptydoublepage
\phantomsection
\addcontentsline{toc}{chapter}{Abstract}

\vspace*{2cm}
\begin{center}
{\Large \textbf{Abstract}}
\end{center}
\vspace{1cm}

Human beings possess the most sophisticated computational machinery in the known universe. We can understand language of rich descriptive power, and communicate in the same environment with astonishing clarity. Two of the many contributors to the interest in natural language - the properties of Compositionality and Structure Dependence, are well documented, and offer a vast space to ask interesting modelling questions. The first step to begin answering these questions is to ground verbal theory in formal terms. Drawing on linguistics and set theory, a formalisation of these ideas is presented in the first half of this thesis. We see how cognitive systems that process language need to have certain functional constraints, viz. time based, incremental operations that rely on a structurally defined domain. The observations that result from analysing this formal setup are examined as part of a modelling exercise. Using the advances of word embedding techniques, a model of relational learning is simulated with a custom dataset to demonstrate how a time based role-filler binding mechanism satisfies some of the constraints described in the first section. The model's ability to map structure, along with its symbolic-connectionist architecture makes for a cognitively plausible implementation. The formalisation and simulation are together an attempt to  recognise the constraints imposed by linguistic theory, and explore the opportunities presented by a cognitive model of relation learning to realise these constraints. 

 \tableofcontents

 \mainmatter


\addtolength{\evensidemargin}{-12mm}

%
%
\chapter{Introduction}
\label{chapter:Introduction}
\section{The Essential Ingredients of Language}
The potential to comprehend natural language is a hallmark of human cognitive ability. Many believe that the faculty for complex thought is tied closely with this ability \cite{martin_modelling_2020}\cite{fodor_connectionism_1988}\cite{frankland_concepts_2020}. Therefore, the modelling of such a system that is tied to what we understand about the formal aspects of language, is an important step towards understanding human cognition itself. \\

Perhaps the most interesting feature of language is the unbounded creativity it offers. Starting from a finite set of symbols, with a primitive set of rules, one can generate a discrete infinity of different structures. Recursion in language is defined as ``the property of a finitely specified generative procedure that allows an operation to reapply to the result of an earlier application of the same operation" \cite{everaert_structures_2015}. \\

A key feature of language that makes this powerful generative procedure interesting is Compositionality. It is the principle that constrains the relation between form and meaning, by requiring that the meaning of a complex expression be built up from the meanings of its constituent expressions and the way they are combined. Importantly, compositionality keeps the meanings of the elements intact, only affecting how the whole complex is interpreted, as a function of the elements and their ordering. \\

While it is true that these properties are essential to language, they can exist independently of each other in other domains. This is because compositionality is the property of units that a system possesses, whereas recursion is the result of functions that are applied on the units that constitute a system. A simple example of a recursive property is displayed by the operations used to generate the Fibonacci sequence. But nothing can be said about compositionality exhibited by the numbers of the Fibonacci series. For an example of recursion in language, take two individuals Tim and Ram, conversing about Sanskrit - Tim says ``I know Sanskrit'', to which Ram replies ``I know you know Sanskrit'', to which Tim replies, ``I know you know I know Sanskrit'' and so on. It is not hard for us to see why, due to the recursive power of language, the conversation does not ever need to end.\\

`Hierarchy' is a loaded term, and associated with close, but not identical interpretations in different fields. An important distinguishing feature of hierarchy is the representational tokening of a complex expression. A natural example is the modern democratic society, where increasing populations of people are represented by officials with proportional responsibilities. Language is analogous in the following way. Oftentimes the most direct way of identifying a constituent is to see if a group of words can be substituted with another word and check if the grammaticality of the sentence remains unaffected \cite{pagin_compositionality_2010} . These substitutions act as tokens encoding proportionally greater amounts of information. In a democracy, these tokens represent groups of people, while in natural language, these tokens represent information. Take for instance the sentence - ``Big dogs bite men''.  ``Big dogs'' can be combined and replaced by ``They'', and the sentence, ``They bite men'' remains grammatically correct. On the other hand, combining ``dogs bite'' to form a constituent results in an awkward combination, and we are left without a clue about what role ``Big'' plays here. \\

This process of combination is termed Merge in linguistics, and is defined as - the computational operation that constructs a new syntactic object Z (e.g.,‘ate the apples’) from already constructed syntactic objects X(‘ate’),Y(‘the apples’), without changing X or Y \cite{everaert_structures_2015}. Notice how even when ``Big dogs'' is called a constituent, the words ``Big'' and ``dogs'' keep their meanings intact. The principled way of analysing constituent structure in language is based on the rules of Syntax, which allows us to connect meaning with form \cite{adger_syntax_2015}. The process of Merge is considered an essential feature of Universal Grammar (UG), which is the genetic capacity for language exhibited by all humans. Interesting to note is how Merge in its basic formulation produces unordered objects, and is an attempt to reconcile word orders of various languages with a common, internal representation, not tied to the modalities of externalisation such as speech and sign language.\\

As mentioned by Pagin and Westerstahl in \cite{pagin_compositionality_2010}, there are multiple flavours of compositionality. The weakest version depends only on the atomic elements and the function that uses them in scope. In the second level version, total meaning depends upon the meanings and operations of the intermediate parts. The weak version relies only upon the immediate sub-parts and the syntactic function combining them. In this work, the aim is to bind ``grammatical meaning" with time, all the while staying true to cognitive processing constraints.

\section{Cognitive Modelling and Processing Compositionality}

In their 1976 paper, Marr and Poggio \cite{marr_understanding_1976} laid out a theoretical framework for the investigation of cognitive processes. In it, they argued for three levels of analysis that should be considered in any cognitive science research - computation, algorithm, and implementation. The first, computational analysis, relates the underlying processes being studied to their mathematical expressions. Second, the algorithmic description involves making clear how the mathematical form relates to the procedural capabilities or mechanisms at hand. The third level consists of the actual implementation of these algorithms on hardware. According to them, each level of analysis should be kept independent, and that the computational level, although often most neglected, should be given primary importance.\\

Guest and Martin \cite{guest_how_2020} argue for formal modelling as an essential step in the process of scientific inference. Most psychological research takes the path of considering a hypothesis and collecting and reporting data relevant to the hypothesis. On the other hand, scientists trained in formal methods alone tend to construct models that are not tied to experimental observations. Using the best of both worlds, by \textit{specifying and modelling} the essential assumptions of a theory will only help guide and effectively constrain further exploratory analyses. \\

Since the primary intention of this work was to connect formal theory and implementation, choosing a suitable cognitive architecture was important. Modern Natural Language Processing techniques are capable of generating text from a given topic, creating captions from images, understanding human speech, and much more, with tremendous accuracy \cite{devlin_bert_2019}. Starting from the days of the simple Perceptron \cite{press_perceptrons_nodate}, we now have large scale neural networks with millions of parameters, that take billions of data points to get to human-like behavior \cite{silver_mastering_2017}. One problem that ails such large neural network models is their inherent lack of explainability. More often than not, such models are treated as black boxes whose measure of improvement relies on small nudges in a large hyperparameter space that is hard to examine in a principled fashion \cite{noauthor_gpt-3_nodate}. An even bigger problem in these models is the lack of an explicit notion of a \emph{Variable}. \\ 

We have another class of models, referred to as symbolic systems, that rely on structured knowledge representations, with modeller intuition guiding the rules that manipulate or transform these representations. Such symbolic systems are known to be powerful but inflexible, and heavily reliant on human intuition.\\

In his book, ``Words and Rules'', Steven Pinker \cite{pinker_words_2011} writes about the existence of a continuum of cognition, one end of which is populated by generative systems that rely on well-defined variables and functions to manipulate them. On the other end exist complex associationist networks that rely on distributed computation capable of modelling non-linear behavior. While Artificial Neural Networks work well to encode statistical regularities of data, they are not suitable proxies of cognitive processes that require compositionality. Recurrent Neural Networks (RNNs) use tensor manipulation operations to represent the internal states of a system. This pays rich dividends when the goal is to minimise a loss function defined on the vector space, but comes at the cost of not knowing where or what each state representation is made of. \cite{martin_tensors_2020}\cite{martin_compositional_2020}.\\

\section{A brief overview of DORA}

The Discovery of Relations through Analogy (DORA) model \cite{doumas_theory_2008}, is a derivative of the earlier LISA (Learning and Inference with Schemas and Analogies) \cite{hummel_distributed_nodate}, and belongs to a class of models termed symbolic-connectionist. DORA was originally intended as a model of human analogical reasoning and possesses a neural network architecture. DORA's learning and analogical inference abilities are very close to what is known from human developmental psychology\cite{doumas_theory_2008}. Starting from a flat n-dimensional encoding of symbols, DORA learns structured representations by a Self Supervised Learning (SSL) algorithm. The simplest unit of abstraction is the predicate/object unit, both identical in function. The details of how DORA learns predicates for relational learning tasks are well documented in \cite{doumas_theory_2008}, but the problem of learning predicates for sentences in natural language is an open one, and we use DORA only to understand the mechanisms of structural mapping and representation. DORA uses dynamic binding to encode propositions. Each proposition is represented by an array of Role-filler binding (RB unit) units that contain information about the functional predicate and the object it is composed of \ref{DORA-arch}. Activation is passed between the Propositional layer and the RB units. These RB units are also bidirectionally connected to the predicate and object units (whose binding each RB represents). The Predicate and Object units (PO units) pass activation to the semantic layer which can be thought to represent perceptual input/output encoders. From the perspective of this thesis, DORA offers an essential advantage over other models by keeping the separation between predicate and object roles in a relationship intact, thereby preserving compositional structure.\\

\begin{figure}
    \centering
    \includegraphics[scale=0.75]{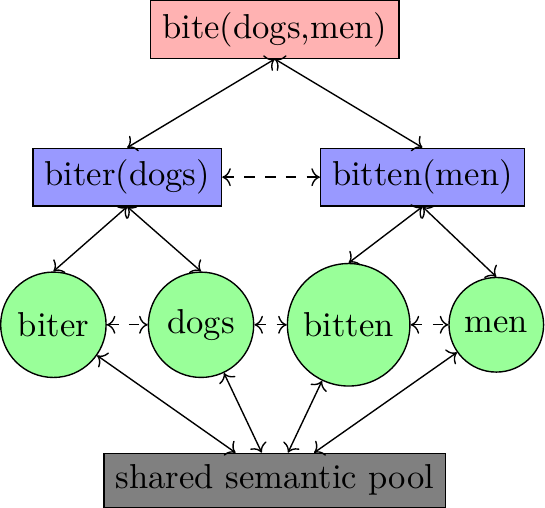}
    \caption{DORA representation of the proposition - bites(dogs,men) in predicate logic. Red unit belongs to P layer, blue - RB layer, green - PO layer, gray - semantic layer}
    \label{DORA-arch}
\end{figure}

This work is an attempt to clearly define the assumptions that we make about the compositional nature of language, and what it means for language to possess grammaticality and structure dependence in chapter \ref{chapter:body1}. In chapter \ref{chapter:body2}, the emphasis is on how the constraints imposed on the system in chapter \ref{chapter:body1} can be realised in a cognitive model, borrowing concepts from Machine Learning (ML) and Natural Language Processing (NLP) Techniques. In chapter \ref{chapter:resultsanddiscussion}, we see the results of simulations run on DORA using a custom dataset, review the limitations of the project, and discuss possible paths for the future.

\chapter{Formalisation}
\label{chapter:body1}
We begin this chapter by going through a few basic mathematical concepts, and then apply them to ideas from formal linguistics.

\section{Basic definitions}

Set Theory serves to make our intuitions concrete and communicable. It is therefore a good starting point in all modelling exercises to define clearly what we are setting out to explore, and the tools we will be using in the process. Set refers to a collection of objects, and Function refers to a way of sending every element of one Set to elements of another Set.

\begin{definition}
Let $X$, $Y$ be Sets. The Coproduct of $X$ and $Y$ is defined as the disjoint union of $X$ and $Y$, $X + Y$ where every element of $X + Y$ belongs either to $X$ or $Y$. If something is an element of both X and Y we include both copies and include information about where the element originates from. 

The Product of $X$ and $Y$ is defined as the set of all ordered pairs $(x,y)$, $X\times Y$ where $x\in X$ and $y\in Y$.
\label{Coproduct_def}
\end{definition}

\subsection{Pushout}

\begin{definition}
An Equivalence relation on $X$ is a subset $Y \subseteq X\times X$ such that elements of $Y$ satisfy the properties of 
\begin{enumerate}
    \item Reflexivity : $(x,x) \in Y, \forall x\in X$
    \item Symmetry : $(x_{1},x_{2}) \in Y \Rightarrow (x_{2},x_{1}) \in Y$
    \item Transitivity : $(x_{1},x_{2}), (x_{2},x_{3}) \in Y \Rightarrow (x_{1},x_{3}) \in Y$
\end{enumerate}
The equivalence relation is indicated by the symbol $\sim$ taken to mean if $x_{1}\sim x_{2}$, then $(x_{1},x_{2})\in Y$. \label{Equivalence-relation_def}
\end{definition}

\begin{definition}
Given the sets X, Y and Z, there is a mapping from $Z$ to the sets $X$ and $Y$ via the functions $f$ and $g$, respectively. The fiber sum or the Pushout is the quotient of $X\cup Z\cup Y$, by the equivalence relation $\sim$. For a pushout, an equivalence relation is generated by $x \sim f(z)$ and $y \sim g(z)$ for all $z \in Z$. \\

The Pushout is denoted by $X\cup_{Z} Y = X\cup Z\cup Y / \sim$. \\

The Pushout $P$ is defined as the set obtained by taking the disjoint union $X + Y$, and finding the elements $x \in X$ and $y \in Y$ such that there exists a $z \in Z$ where $f(z) = x$ and $g(z) = y$.
\label{Pushout_def}
\end{definition}

 The functions $i_{1}$ and $i_{2}$ are the canonical inclusion functions which map the elements of $X$ and $Y$ to their equivalence classes. These maps $i_{1}$ and $i_{2}$ result in a commutative square, such that $f \circ i_{1} = g \circ i_{2}$. This allows a definition of equivalence relation on $X\cup Z\cup Y$, and $P$ is the quotient of it \cite{spivak_ologs_2012}\cite{phillips_sheavinguniversal_2020}.

\begin{example}
An undirected graph G is defined as a collection of Vertices and Edges, $G = (V,E)$. A vertex $v$ is reachable from a vertex $u$ if there is a path from $u$ to $v$. Reachability is then an equivalence relation :
\begin{enumerate}
    \item Reflexivity is satisfied, since every vertex is reachable from itself.
    \item If there is a path from $u$ to $v$, there is also a path from $v$ to $u$ since the graph is undirected. Therefore symmetry is satisfied.
    \item If there is a path from $u$ to $v$ and a path from $v$ to $w$, the two paths can be joined so that a path from $u$ to $w$ exists. Therefore transitivity holds
\end{enumerate}

The quotient of reachability on $V$ gives us the set of connected components, or cliques, in $G$.
\end{example}

\subsection{Universal Property of Pushouts}

\begin{definition}
The universal property of Pushouts states that given any commutative square with a pushout such as the one shown in fig \ref{univ_pushout}, and given another set $Q$ for which the mapping from $Z$ to $Q$ commutes ie., $f \circ j_{1} = g \circ j_{2}$, there must exist a unique $P \to Q$, also making the diagram commute, as shown in fig \ref{univ_pushout}. 
\end{definition}

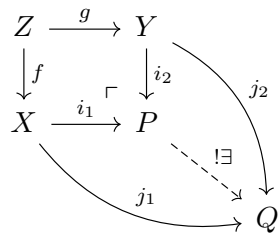
\begin{figure}[h!]
    \centering
    \begin{tikzcd}
    Z \arrow[r, "g"] \arrow[d, "f"] 
    &Y \arrow[d,"i_{2}"] \arrow[ddr, bend left,"j_{2}"]\\
    X \arrow[r, "i_{1}"] \arrow[drr, bend right,"j_{1}"]
    &P \arrow[ul, phantom, "\ulcorner", very near start] \arrow[dr, dashed,"!\exists"] \\
    &&Q
    \end{tikzcd}
    \caption{Universal property of pushouts}
    \label{univ_pushout}
\end{figure}

In this thesis, we consider the universal property on partial functions, and the composition of partial functions is not the same as the composition of (total) functions. For example, the mapping $X \to Q$ can be a partial function in that it does not map every element of $X$ to an element in $Q$, but only a subset of $X$ has an image in $Q$. The necessity for partial functions when dealing with compositionality is also stressed in \cite{pagin_compositionality_2010}, although the reason for such a distinction is that grammatical understanding may be imperfect or incomplete. In our case, we treat the existence of the partial functions as a fundamental necessity to account for novel sequences. Under such a condition, we can reformulate the composition of functions in terms of inclusions and restrictions of (total) functions as shown below \cite{simmons_introduction_nodate}. 

\subsubsection{Treatment of partial functions}

Given three sets $S,T,H$, and a partial function $f:S\to T$ and another partial function $h:T\to H$. That is, not all elements of $S$ have an image in $T$, and not all elements of $T$ have an image in $H$. To represent a partial function as a total function, we consider subsets $S^{'}\subset S$, $T^{'}\subset T$, such that $\overline{f}:S^{'}\to T$, $\overline{h}:T^{'}\to H$ are full functions, as shown in fig \ref{def_pfn}. 

\begin{figure}[h]
    \centering
    \begin{tikzcd}
    S \arrow[r,"f"] &T &&T \arrow[r,"h"] &H\\
    S^{'} \arrow[u,hook] \arrow[ur,"\overline{f}",labels=below right] &&&T^{'} \arrow[u,hook] \arrow[ur,below,"\overline{h}",labels=below right]
    \end{tikzcd}
    \caption{Partial functions defined on subsets}
    \label{def_pfn}
\end{figure}
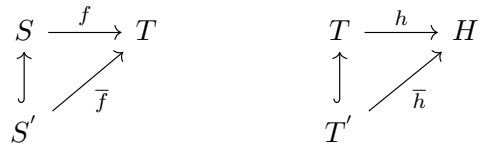

We now take a subset $U\subseteq S^{'}$ such that if $s\in U \iff s\in S^{'}$ and $\overline{f}(s)\in T^{'}$, as in fig \ref{comp_pfn}. These functions are now composition compatible, and we take $h\circ f$ to be determined by $\overline{h \circ f} = \overline{h} \circ \overline{f_{|U}}$. That is, the composition of partial functions is defined as a restriction on total functions. This exercise will not be explicitly shown in later sections due to space constraints. In the following sections, compositions are taken to mean compositions of total functions determined from partial functions.

\begin{figure}[h]
    \centering
    \begin{tikzcd}
    S \arrow[r,"f"] &T \arrow[r,"h"] &H\\
    S^{'} \arrow[u,hook] \arrow[ur,"\overline{f}",labels=below right] &T^{'} \arrow[u,hook] \arrow[ur,"\overline{h}",labels=below right] \\
    U \arrow[u,hook] \arrow[ur,"\overline{f_{|U}}",labels=below right]
    \end{tikzcd}
    \caption{Partial functions composed}
    \label{comp_pfn}
\end{figure}
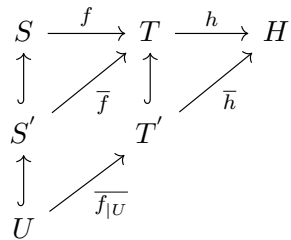

One condition that results from considering such a composition of partial functions is that although commutativity holds, the uniqueness condition is not guaranteed. Therefore, from the stronger ``unique existence" from $P$ to $Q$, we have ``existence" only, giving us the ``weak pushout''. 

\section{Specifications}

\begin{definition}
Let $L_{0}$ be the set of all words in a natural language $Language_{k}$. 
\label{vocabulary}
\end{definition}

\begin{definition}
$S$ = \{$w + v | w, v \in L_{0}\cup S$\} The + operator denotes usual string concatenation, in addition to a space in between. 
\label{sequences}
\end{definition}

To make understanding easy, we start with a small subset of English words. \\

$L_{0}$ = \{`Men', `Dogs', `bite', `say', `big'\} \\

We combine words at random order to get sequences of all lengths. Ex :"Dogs dogs men men", ``Dogs men bite'', ``Big men bite dogs'', ``Dogs say men bite men'', and so on. Recursion is taken care of in \ref{sequences} because we combine elements of $S$ to generate new elements of $S$. Use of the verb `say' lets us create grammatical sentences of arbitrarily increasing lengths (In linguistics, a complementizer is a word that turns a clause into the subject or object of a sentence, and strictly speaking, one must use the complementizer ``that" with ``say", but for the sake of simplicity, without sacrificing grammaticality, we do not). Example - ``Dogs say big dogs say men say big dogs bite men'' etc.,\\

\begin{definition}
Merge is an essential binary operation that exists in language. Merge takes two syntactic objects as input, and constructs a third object which is a complex whose existence does not destroy the meaning of the parts it is made up of. Simply put, Merge(X,Y) gives a set \{X,Y\}.  \cite{everaert_structures_2015}.
\label{Merge_def}
\end{definition}

To incorporate constituents into the formulation, we designate embedded constituent sets by $L_{i}$, which is defined as 
\begin{definition}
Given that $f_{1}$ is the syntactic projection on words which identifies its lexical category ($f_{1}(big) = Adjective$ etc.)
\[   
    L_{i} = 
         \begin{cases}
           \text{\{$l + h$ $ | $ $l \in L_{m}, h \in L_{n}$; $max(m,n)$ $\leq (i-1)$ \}} &\quad\text{if i $>$ 1}\\
           \text{\{$f_{1}(l)$ $|$ $l \in L_{0},$ and $f_{1}$ is the syntactic projection on  $l$\}} &\quad\text{if i $=$ 1} \\
         \end{cases}
    \]
\label{constituent_def}
\end{definition}

In \ref{constituent_def}, `+' is the Merge operator on two strings, resulting in a new, valid constituent. With this operation, for  i $>$ 1, every $L_{i}$ consists of two constituents, at least one of which belongs to $L_{i-1}$. For a detailed example, see \ref{detailed_1}.\\

For our purpose, we only consider $L_{i}$ where $i > 1$ in further steps, since $i=1$ does not include any compositional information. In this formalisation, $L_{i}$ should be read as ``Constituent set i''.  The creation of elements of $L_{1}$ from $L_{0}$ is a linguistic projection, which assigns to words syntactic categories that lets them bind with other constituents.\\

Let us take the example of a sentence from S - ``Dogs say big dogs bite men'', whose constituent tree structure (see \cite{adger_syntax_2015} for a neat introduction) is as shown in fig \ref{tree}. Deep structure (the underlying logical relationships of the elements of a phrase or sentence), of the kind seen in \ref{tree} is a result of recursive Merge operations performed on the sequence of words. We see that `bite' combines with `men' to form the verb phrase `bite men', which belongs to $L_{2}$, and $L_{2}$ also contains the constituent ``big dogs''. These two constituents combined, result in elements of $L_{3}$ according to the rules defined above. This takes place until we reach the root of the tree, which is the full sentence, and also a member of $L_{5}$.\\

\begin{figure}[h]
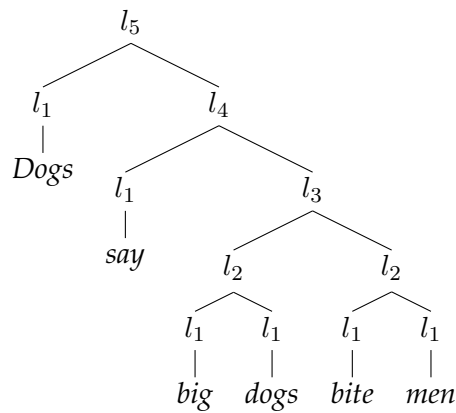

\Tree[.$l_{5}$ [.$l_{1}$ \textit{Dogs} ]
            [.$l_{4}$ [.$l_{1}$ \textit{say} ]
                  [.$l_{3}$ 
               [.$l_{2}$  [.$l_{1}$ \textit{big} ][.$l_{1}$ \textit{dogs} ]]
                    [ .$l_{2}$ [.$l_{1}$ \textit{bite} ] 
               [.$l_{1}$ \textit{men} ] ]
                    ]]]
\caption{Tree structure born out of recursive Merge}
\label{tree}
\end{figure}

What constitutes a sentence is the presence of a root node that subsumes all constituents present in the structure. Instances of sequences that possess grammatically valid constituents but cannot be called sentences are shown in \ref{disjoint_trees}. \\

\begin{figure}[h!]
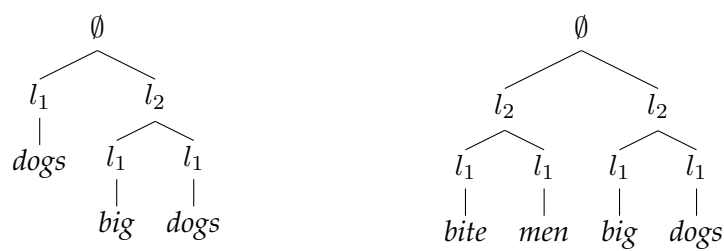

\Tree[.$\emptyset$ [.$l_{1}$ \textit{dogs} ]
            [ .$l_{2}$ [.$l_{1}$ \textit{big} ] 
            [.$l_{1}$ \textit{dogs} ] ]]
\Tree[.$\emptyset$ [.$l_{2}$ [.$l_{1}$ \textit{bite} ] 
                          [.$l_{1}$ \textit{men} ]]
            [ .$l_{2}$ [.$l_{1}$ \textit{big} ] 
            [.$l_{1}$ \textit{dogs} ] ] ]
\caption{Invalid constructions}
\label{disjoint_trees}
\end{figure}

The rules for determining the validity of a constituent are language-specific, and word order plays a part, specifically when considering languages in isolation. We take a South Indian language, Kannada, belonging to the Dravidian language family to make this idea clear. Linguistic typology identifies English as possessing the S(subject)-V(verb)-O(object) order for active forms, and Kannada possesses the S-O-V order. Take a look at how these forms vary in Kannada with the translation of the sentence - "Dogs say big dogs bite men" in \ref{tree_kannada}. \\

\begin{figure}
    \centering
    \includegraphics{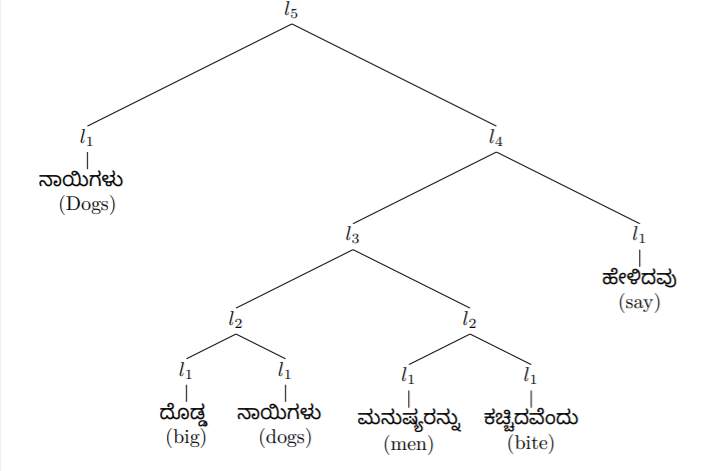}
    \caption{Tree structure in Kannada}
    \label{tree_kannada}
\end{figure}

\section{Grammaticality}

\begin{definition}
$T$ is the set of all positions in a sequence that a constituent can occupy. Therefore $T=\mathbb{N}$
\end{definition}

\begin{definition}
$L^{*}_{i} =  L_{i} \times T = \{ (l, t) | l \in L_{i} , t \in T$\} 
\label{l_i_*}
\end{definition}

\begin{definition}
$P^{*}_{i} = \mathbb{P}(L^{*}_{i}) \times \{0,1\} $, $\mathbb{P}$ is the power set operator. 
\end{definition}

$T$ is the set of all time positions that a constituent can occupy. The product set $L^{*}_{i}$ is defined to associate a constituent with every possible position it can occupy. With \ref{l_i_*}, we already make room for the discrete infinity of recursively created constituents that language allows. The power set of this product is defined as $P^{*}_{i}$ (which includes the $\emptyset$ as its element) to create the possibility of mapping a sequence to the set of constituents it contains. The magic term $\theta$ that appears along with the power set will be described below. \\

Labelling of non-terminal nodes of a syntactic tree is a way to identify compositional structures' content by the abstract lexical features that encode constituents. Primarily, lexical (Noun - N, Verb - V, Adjective - Adj, Preposition - P) and functional (explained under \ref{def_X-Bar})categories encode word-level features, with phrasal categories encoding higher-order combinations (Ex : Noun Phrase - NP, Verb Phrase - VP, Inflectional Phrase - IP, Prepositional phrase - PP etc., ). The label `IP' stands for Inflectional Phrase and is in all respects identical to S, except in that it can be treated as possessing a Head, Complement, and Specifier(which will be explained later). Words can be thought to contain two (approximately) orthogonal dimensions encoding lexical and semantic information. The composition of two words results in a complex, higher-order lexical feature, which is a result of the Merge operation. 

\begin{definition}
$L_{syn}$ = \{N, V, Adj, P, NP, VP, IP, PP\}. The small vocabulary set, $L_{0}$ affords us a small set of syntactic categories to work with. 
\label{L_syn}
\end{definition}

\begin{definition}
Let $s\in S$ and $s = (a,b,c...)$, where $a,b,c... \in L_{0}$\\
$f_{i} : f_{i}(s,f_{i-1}(s),f_{i-2}(s),...f_{2}(s),f_{1}(s)) = (p^{*}_{i},\theta),$ where $p^{*}_{i}\in \mathbb{P}(L^{*}_{i})$, and $s\in S$
\label{def_f}
\end{definition}

The function mapping a sequence $s$ to a level set $P^{*}_{i}$ needs to be able to distinguish between the various configurations that are possible within the constituents that are part of its building blocks (see \ref{detailed_1}). The Binary Branching Hypothesis (BBH) is a minimal assumption that considers every constituent to be capable of dominating at most two other constituents \cite{kayne_antisymmetry_1994}. Given BBH, identifying a constituent structure is influenced by the identity of the two constituents that make it up. Therefore, a principled way of mapping must rely not only on the sequence itself but also on the previously (hierarchically lower) mapped constituents. The functional definition in \ref{def_f} does just that. These functions rely both on the lexical and semantic features encoded by the elements of $s$. It should therefore be assumed that every one of these functions also takes as an argument, implicitly, the lexico-semantic features of $e \in s$ itself.\\

This mapping characterises the bottom-up nature of the perceptual inference of linguistic input in two ways. Firstly, the accumulation of evidence for constituency is time-bound. For example - given an incomplete sequence $s^{'} = (a,b)$, which is a part of $s = (a,b,c...)$, this model is capable of forming constituents using what is known at any given point in time. Secondly, this information needs to build upon what is already known about the computed constituents, thereby constraining the processing structure cognitive models employ. The implementational definition of these functions will be explicitly analysed in the section dealing with DORA in \ref{part:body2}. \\

Consider the function $f_{2}:S\times P^{*}_{1}\to P^{*}_{2} $, which maps every element in the product of S and $P^{*}_{1}$ to the set-of-set-of ordered pairs of level two constituents. For example $f^{'}_{2}$(``Big dogs bite men'') = \{((``Big dogs''),(``bite men'')), 0\}. This means the sentence "Big dogs bite men" has two $L_{2}$ constituents, "Big dogs" and "bite men", which end at the second and fourth position respectively. $\theta$ simply is the Boolean value indicating whether or $p^{*}\in P^{*}_{i}$ is the highest constituent level achievable by $s$. We use $\theta$ in order to determine if the root node is present, and its value dictates grammaticality. In case of \{((``Big dogs''),(``bite men'')), 0\}, the `0' indicates that the root node is absent in the $L_{2}$. \\

\begin{definition}
$T^{*}_{i} = \mathbb{P}(T)\times \mathbb{P}(L_{syn}) \times \{0,1\}$. $\mathbb{P}$ is the power set operator, $L_{syn}$ is as defined in \ref{L_syn}, and $\theta\in \{0,1\}$.
\label{def_t_star}
\end{definition}

The intuition behind this step is to abstract away from a sequence of words to time patterns of constituent formation. This way, we can picture a series of functions that map from S to all the sets $P^{*}_{i}$. We obtain the composition of functions $f_{i} \circ \pi_{i}$ (where $\pi_{i}$ is the natural projection function $P^{*}_{i} \to T^{*}_{i}$) to get from $S$ to $T^{*}$, as shown in fig \ref{StoT}. 

\begin{definition}
$Q = T^{*}_{2}\cup_{S} T^{*}_{3}...\cup_{S} T^{*}_{i}$
\label{def_Q}
\end{definition}

The quotient of the disjoint union of all the $T^{*}_{i}$ is the pushout $Q$ (via the inclusion functions $i_{n}$). Notice that the creation of equivalence classes is due to partial functions. For example, $s$ = "Big dog big dog" has two $L_{2}$ constituents, but $f_{3}(s) = \emptyset$. Therefore, it falls into a different equivalence class in the pushout, which does not contain the sentence ``Big dogs bite men''.  Compositional structures in natural language can be analysed in this manner, and its significance from a cognitive processing perspective will be discussed in \ref{part:resultsanddiscussion}.\\

\begin{example}
 Take the mathematical expression : $((4*(3+2) - 3^2))/4 + 0.1$. Following the rules of evaluating regular arithmetic expressions, we obtain the tree structure as in \ref{math_tree}, (with the nodes of the parsed expression replaced by variables). Notice how we do not need to know what the subexpressions evaluate to, but only need knowledge of the underlying structure to evaluate the overall expression. In this example, and for evaluating arithmetic expressions in general, the knowledge of what symbols such as $/,*,(),$, etc., mean forms the ruleset for determining the structure of the composition.  
\label{math_ex1} 
\end{example}

\begin{figure}[h]
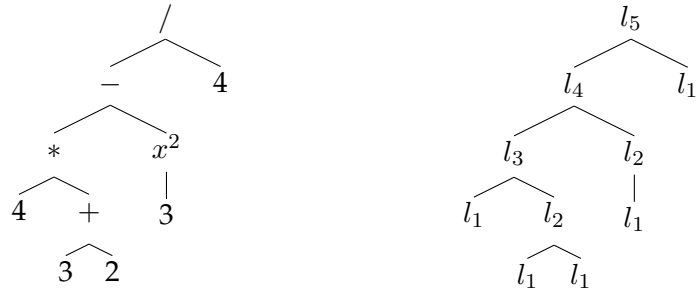

\Tree[.$/$  [.$-$  [.$*$ [.4 ]
                [ .$+$ [.3 ] 
                [.2 ] ]]
                [.$x^{2}$ [.3 ]]]
            [.4 ]]
\Tree[.$l_{5}$  [.$l_{4}$  [.$l_{3}$ [.$l_{1}$ ]
                [ .$l_{2}$ [.$l_{1}$ ] 
                [.$l_{1}$ ] ]]
                [.$l_{2}$ [.$l_{1}$ ]]]
            [.$l_{1}$ ]]
\caption{Evaluating expression in \ref{math_ex1}}
\label{math_tree}
\end{figure}

\begin{definition}
$\mathbb{H}$ = Set of all syntax trees in $Language_{k}$
\end{definition}

Here we define $\mathbb{H}$ as the set of all possible trees generated by English grammar. A branching structure such as the one shown in \ref{tree2} is a result of the knowledge of a language (in this case English), shared by all of its speakers. Therefore, the meaning or value \cite{pagin_compositionality_2010}\cite{hodges_formal_2001} that a speaker or listener ascribes to a sequence of words is dependent not only on surface-level rules, but the complex hierarchical structure generated by constituents. For example, speakers of English have no problem understanding the sentence, ``I watched a gigantic golden armadillo lecture about ballpoint pens", but will unanimously agree that the sentence ``The monkeys of New Delhi is rowdy" is ungrammatical, citing the lack of agreement between the plural ``monkeys" and singular ``is" to be the sticking point. On the other hand, people who are familiar with the idea of Facebook pages will take no issue with, ``The Humans of New York is a fantastic page, filled with heartwarming stories". \\

We can define a mapping from the $T_{i}^{*}$s to $\mathbb{H}$, so that $f_{i} \circ \pi_{i} \circ h_{i}$ commute (see \ref{detailed_1}). This is again a partial function, which takes as input $T_{i}^{*}$, and maps it to a tree in $\mathbb{H}$. This is where the value of $\theta$ comes into play. Every element of $T_{i}^{*}$ which contains the root node (Such that $\theta = 1$) maps to a unique syntactic tree, all elements of $T_{i}^{*}$ where $\theta = 0$ have no corresponding object in $\mathbb{H}$. In the style of \cite{pagin_compositionality_2010}, such a system can be called weakly compositional. Therefore, the topmost syntactic operation (for which $\theta$ is a proxy), along with the structures of its intermediate constituents determines grammaticality. Example - ``Dogs bite men'' has one $L_{3}$ constituent which maps to a binary branching structure as in fig \ref{tree2}.\\

\begin{figure}[h]
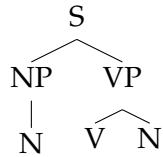

\Tree[.S [.NP [.N ]]
            [ .VP [.V ] 
            [.N ] ]]
\caption{Sample element of $\mathbb{H}$}
\label{tree2}
\end{figure}

The universal property of Pushouts says that there must be a function mapping the pushout $Q$ to $\mathbb{H}$. Note that $|\mathbb{H}| < |Q|$, meaning there are fewer grammatically valid structures than possible ones. As we abstract away from surface order and associate meaning with constituents, we eventually discover that there exist hierarchical structures that are not apparent from linear sequences. In other words, the discovery of grammar is possible given the knowledge of the nature of compositional structures.\\

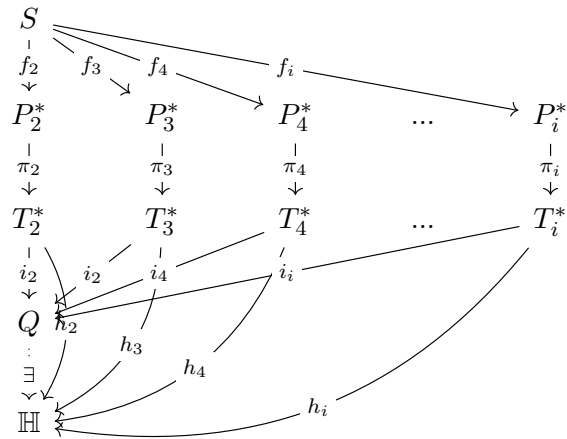
\begin{figure}[h]
\centering
\begin{tikzcd}
S \arrow[d, "f_{2}" description] \arrow[rd, "f_{3}" description] \arrow[rrd, "f_{4}" description] \arrow[rrrrd, "f_{i}" description]\\
P^{*}_{2} \arrow[d, "\pi_{2}" description] &P^{*}_{3} \arrow[d, "\pi_{3}" description] &P^{*}_{4} \arrow[d, "\pi_{4}" description] &... &P^{*}_{i} \arrow[d, "\pi_{i}" description]\\
T^{*}_{2}\arrow[d, "i_{2}" description] \arrow[dd, bend left,"h_{2}" description] &T^{*}_{3}\arrow[ld, "i_{2}" description] \arrow[ldd, bend left,"h_{3}" description] &T^{*}_{4}\arrow[lld, "i_{4}" description] \arrow[lldd, bend left,"h_{4}" description] &... &T^{*}_{i}\arrow[lllld, "i_{i}" description] \arrow[lllldd, bend left,"h_{i}" description]\\
Q \arrow[d, dotted,"\exists" description]\\
\mathbb{H}
\end{tikzcd}
\caption{Universal property of Pushouts applied to $\mathbb{H}$}
\label{StoT}
\end{figure}

\subsubsection{Discovery of structure}
Example \ref{math_ex1} is similar to natural language in the sense that it displays compositional structure. But language is very different because the symbols it uses also act as markers that organise structure, while arithmetic expressions are evaluated by explicitly defined structures. This makes the task of identifying and separating function from variables in language a hard one for cognitive models.\\

The membership of set $\mathbb{H}$ is dictated by the grammatical rules that \textbf{form and constrain} the possible structures exhibited by a language. These rules can be considered a reduction in the complexity of possible structures. The guiding principle of grammar therefore becomes a position and feature-based exclusion of impossible structures. For example, it can be certainly said given merge rules and our vocabulary that $h \notin \mathbb{H}$, where $h$ is shown in \ref{invalid_tree}.\\

\begin{figure}[h]
    \centering
    \includegraphics[scale=0.5]{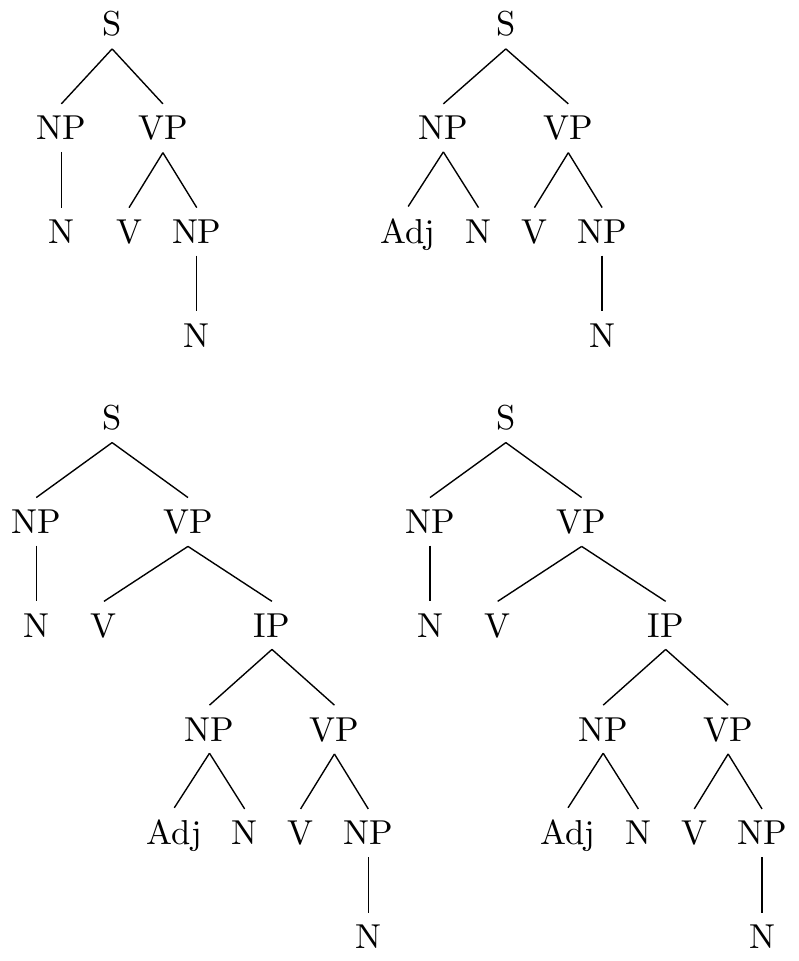}
    \caption{Sampled elements of $\mathbb{H}$}
    \label{fig:h_trees}
\end{figure}

Compositional structures are more informational than linear sequences. Given a linear sequence of $k$ symbols and $n$ slots, we can arrange them in $n^{k}$ ways. But from a binary branching viewpoint, if every slot is considered the leaf node of a tree, we get $C_{n-1}$ number of unique trees for a given sequence of symbols, where $C_{n-1}$ is the Catalan number $C_{n} = \frac{2n!}{n!(n+1)!}$. The rise in complexity of a sequence treated as a linear combination of strings versus its complexity defined by the compositional structure it contains, is $C_{n-1}$. (The sequence of $C_{n} = 1, 1, 2, 5, 14, 42, 132, 429, 1430, 4862, 16796, 58786, 208012$.. for $n = 0, 1, 2, 3 ..$). But the story does not end here, since words possess latent lexical content depending on position. For example, although the word `man' is a noun, if placed in front of `eater', it acquires the characteristics of an adjective, and the phrase becomes a noun phrase, with further compositions distinct from what treating `man' as a noun would have entailed. Therefore, it can be said that a grammar both enriches and restricts the existence of valid structures.

\begin{figure}[h]
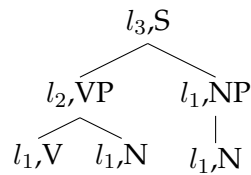

\Tree[.$l_{3}$,S [.$l_{2}$,VP [.$l_{1}$,V ]
              [.$l_{1}$,N ] ]
         [.$l_{1}$,NP [.$l_{1}$,N ] ] ]    
\caption{Grammatical rule as exclusion, $h\notin \mathbb{H}$}
\label{invalid_tree}
\end{figure}

\section{Structure Dependence}
Natural language possesses many distinguishing features that set it apart from other compositional systems. Strong structure dependence is one of them. Compositional systems in general show hierarchy, and use transformations and functions that rely on latent structure rather than purely sequential, operations. The nature of these functions is the subject of discussion in this section. We consider two processes where this property of structure dependence comes to the fore.

\subsection{Interpretation from structure}
Take the example of the sentence ``Dogs bite men with teeth". This sentence can mean one of two things, as shown in \ref{two_interpretations}. In one case, the prepositional phrase (PP) can modify the Noun Phrase (NP) dogs, to give the unambiguous paraphrasing - dogs that have teeth, bite men \ref{interpretation_1}. On the other hand, the prepositional phrase (PP) ``with teeth" can be taken to modify the NP `men', to give ``men with teeth", and the unambiguous sentence - dogs bite men who have teeth \ref{interpretation_2}. \\

This can be compared with evaluating arithmetic operations like $s = 3 * 2 + 9 - 5 / 4$, which can result in 2.5 if evaluated left to right or 13.75 if the order of operations convention (BODMAS - Brackets of - multiplication - addition - subtraction) is adopted. The conventions that guide the evaluation of such expressions are often well defined, and differences are trivially solved by agreeing upon conventions. We cannot assume that all compositional systems display structure dependence of this variety. \\

\begin{figure}[h]
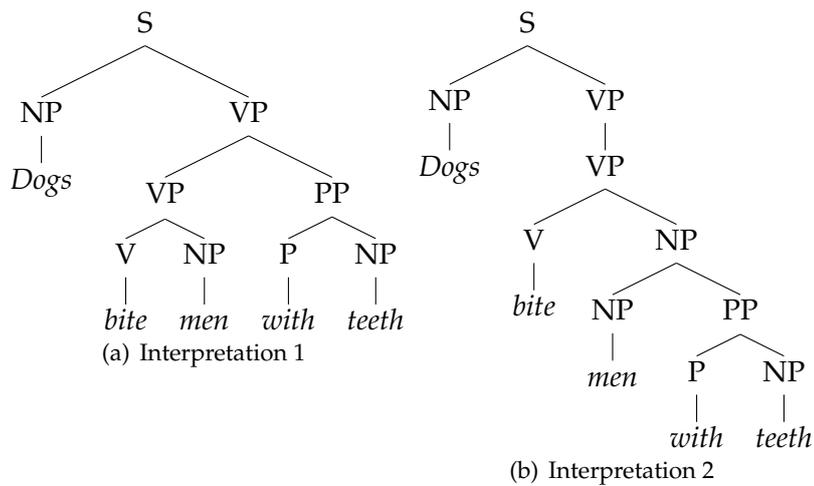

\centering
\subfigure[Interpretation 1]{\label{interpretation_1}
    \Tree[.S    [.NP \textit{Dogs} ]
            [.VP   [.VP    [.V \textit{bite} ]
                            [.NP \textit{men} ]]
                    [.PP    [.P \textit{with} ]
                            [.NP \textit{teeth} ]]]]
            }
\subfigure[Interpretation 2]{\label{interpretation_2}
    \Tree[.S    [.NP \textit{Dogs} ]
            [.VP   [.VP    [.V \textit{bite} ]
                            [.NP    [.NP \textit{men} ]
                                    [.PP    [.P \textit{with} ]
                                            [.NP \textit{teeth} ]]]]]]
            }
\caption{Interpreting ambiguous sentences}
\label{two_interpretations}
\end{figure}

\subsection{Transformation of structure}

In linguistics, Wh-movement refers to the formulation of interrogative sentences and the placement of the Wh words (Who, what, etc., ) to refer to the constituent being queried.\\

This surprising and counter-intuitive phenomenon is often not recognised by non-linguists. Take the example of the sentence - ``I am writing my thesis''. If someone wanted to know the answer to the blank -``I am writing \tikz[baseline]\draw[dashed](0,0)--+(1,0);'', they would ask - ``\underline{What} are you writing?'' \textbf{instead of} ``Are you writing \underline{what}?" - 

\begin{longtable}{|p{3cm}|p{\linewidth-3cm}|}
\hline
Term & Description \\[0.5ex] 
\hline\hline
Merge & It is a binary operation on two syntactic objects X,Y to construct a third object Z, which is simply the set X$\cup$ Y. \\
Syntactic projection (as used here) & The lexical label given to a word\\
Deep Structure &  The underlying logical relationships of the elements of a phrase or sentence\\
Externalisation & The mapping from internal linguistic representations to their ordered output form, either spoken or manually gestured \\
Move & The transformation on Deep Structure which rearranges word order for externalisation\\
Wh-movement & Is a phenomenon where the interrogative word (Who, what, etc.) appears at a different position from the answer.\\
Head & Within a constituent, there is usually a single element which is of primary importance in determining the grammatical behavior of the whole unit. Such a word or phrase is called the head. Ex - In ``The temples of India'', the word ``temples'' determines the agreement on the verb which comes after the phrase, and is, therefore, the head\\
Xbar theory & A structural setup for defining a phrase, which contains a head. The main idea behind the development of this theory is that all phrase structure can be reduced to simple recursive operations on basic configurations containing the head. See \ref{def_X-Bar}\\
Complement & The phrasal constituent which is the sister (Two sister constituents are inputs to a Merge operation) of a head\\
Specifier &  The phrasal constituent which is the sister of the mother (the product of Merge operation on two sisters) of the head\\
\hline
\caption{Commonly used linguistic terms \cite{everaert_structures_2015},\cite{adger_syntax_2015},\cite{chomsky_syntactic_2002}}
\label{linguistic_def}
\end{longtable}

This is a direct example of what is meant by strong structure dependence, where the generation of wh-questions and the position of the wh-word along with its referents are dependent upon the structure being queried, and not tied solely to rigid, surface-level dependencies. We make this distinction formal by identifying the domains of structure-dependent operations and rigid, structure independent operations as in \ref{dependent_vs_independent}. 

\begin{figure}[h]
\centering
\begin{tikzcd}
S \arrow[d, "f_{2}" description] \arrow[rd, "f_{3}" description] \arrow[rrd, "f_{4}" description] \arrow[rrrrd, "f_{i}" description] \arrow[rrrrr, dotted,"g_{si}"] &&&&&G\\
P^{*}_{2} \arrow[d, "\pi_{2}" description] &P^{*}_{3} \arrow[d, "\pi_{3}" description] &P^{*}_{4} \arrow[d, "\pi_{4}" description] &... &P^{*}_{i} \arrow[d, "\pi_{i}" description]\\
T^{*}_{2}\arrow[d, "i_{2}" description] \arrow[dd, bend left,"h_{2}" description] &T^{*}_{3}\arrow[ld, "i_{2}" description] \arrow[ldd, bend left,"h_{3}" description] &T^{*}_{4}\arrow[lld, "i_{3}" description] \arrow[lldd, bend left,"h_{3}" description] &... &T^{*}_{i}\arrow[lllld, "i_{i}" description] \arrow[lllldd, bend left,"h_{i}" description] \arrow[uur, bend right,dotted,"g_{sd}"]\\
Q \arrow[d, dotted,"!\exists" description]\\
\mathbb{H}
\end{tikzcd}
\caption{Distinguishing structure dependent ($g_{sd}$) and structure independent ($g_{si}$) operations, defined in \ref{defintion_structure_dependence}}
\label{dependent_vs_independent}
\end{figure}
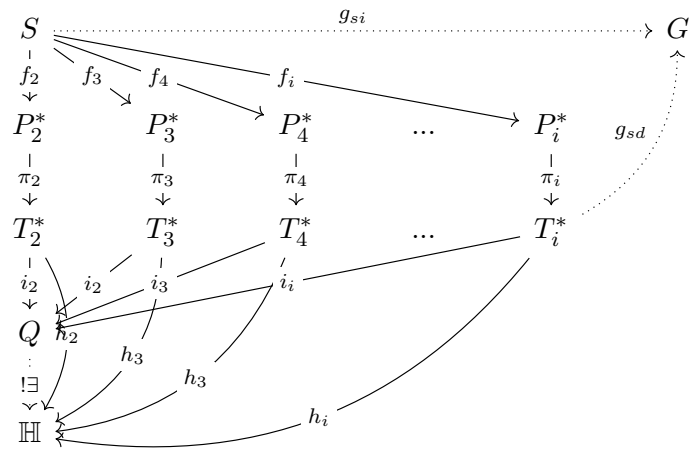

Returning to our world where dogs speak, we have two valid grammatical sentences - ``Dogs are biting men", and ``Dogs said big dogs are biting men" (For ease of explainability, we add the words ``are", ``biting", ``said" to our original vocabulary). The detailed tree structure for the first sentence is drawn in \ref{struct_dependence_1}. 

\begin{figure}[h]
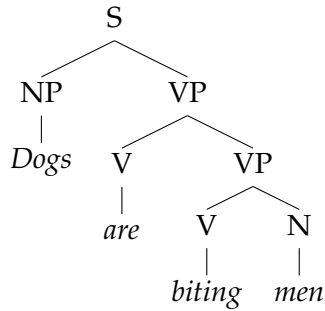

\Tree[.S    [.NP \textit{Dogs} ]
            [.VP    [.V \textit{are} ]
                    [.VP [.V \textit{biting} ]
                     [.N \textit{men} ]]]]
\caption{Structure dependence in action}
\label{struct_dependence_1}
\end{figure}

To find the object of the verb `bite', which in this case is ``men", we replace the object noun - ``men", with the pronoun - ``who", to get the interrogative sentence ``Dogs are biting who". But this is not how the sentence is externalised, or spoken/gestured. The transformation on the sequence - ``Dogs are biting who" to "Who are dogs biting?" is a result of applying a \textbf{structure (or non-rigid)} operation on the original sentence ``Dogs are biting men". The operations necessary for this are explained below.

\subsection{Merge and Move}

We have defined the Merge operation in \ref{Merge_def} and seen how it is essential to creating compositional Deep Structure. Move is the other half of the problem, which relates the transformation of deep structure to surface order. To further understand how they function, we take a slight detour into a few linguistic concepts, restricting ourselves to the bare minimum of what is necessary and within the scope of this thesis.

\begin{definition}
The X-Bar theory proposes a basic structural setup for a phrase in a sentence, with every phrase possessing a head, and also may contain other phrasal constituents in the specifier and/or complement position.

If the phrasal constituent is the sister of a head, it is called the complement, and if a phrasal constituent is the sister of the mother of the head, it is called the specifier. The X-Bar structure of a typical phrase is shown in fig \ref{Xbar}.
\label{def_X-Bar}
\end{definition}

\begin{figure}[h]
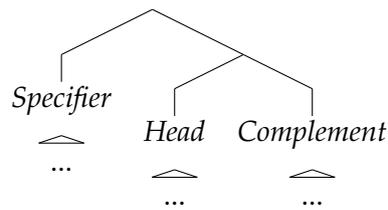

\Tree[.    [\qroof{...}.\textit{Specifier} ]
            [.  [\qroof{...}.\textit{Head} ]
                [\qroof{...}.\textit{Complement} ]]]
\caption{X-Bar structure}
\label{Xbar}
\end{figure}

The proposal that there exists a head in every phrase has led to the analysis of sentences considering strictly ordered functional categories. For example, Tense (T), which marks the grammatical tense, is a functional category and differs from lexical categories such as Nouns, Verbs, etc., in several ways \cite{adger_syntax_2015}. 

\begin{definition}
Move is the operation that takes as input Deep Structure generated by the Merge operation, and rearranges constituents to enable externalisation.
\label{Move_def}
\end{definition}

\ref{Move_def} implies that the Head can be moved to other Head positions in the deep structure, as we see in the following example, to generate yes-no questions (To be contrasted with wh-questions). \\

\begin{figure}[h]
\centering
\includegraphics{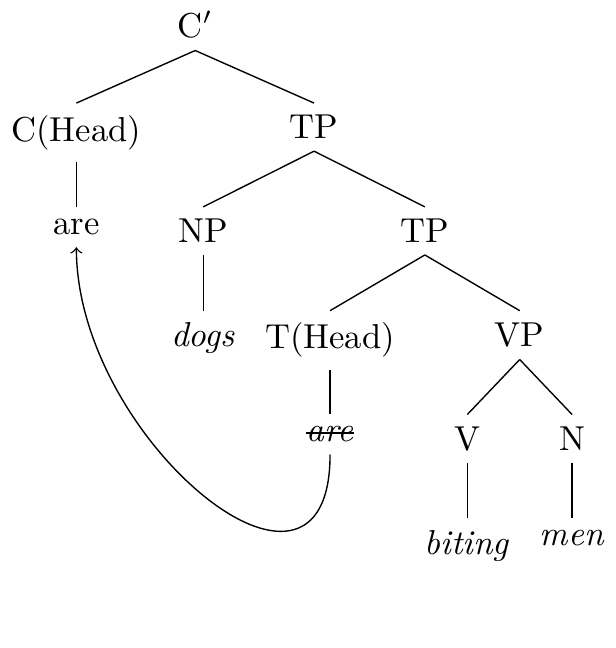}
\caption{Move operation applied to generate yes-no question}
\label{yes-no_q}
\end{figure}

In figure \ref{yes-no_q} the word `are' occupies the head position and has the Tense label. To transform the sentence from a declarative to an interrogative form, we move the `are' to the head complement position C.\\

When generating questions of the Wh-form \cite{anderson_810_2018}, the rules for a move operation is slightly more complex, making the ease with which we judge grammaticality even more astonishing. In this case, depending on what the constituent being queried is, the T head moves to the complement head position, and the Wh constituent moves to the specifier position in \ref{wh_q}(see \cite{anderson_810_2018} for a detailed explanation). \\

\begin{figure}[h]
\centering
\includegraphics{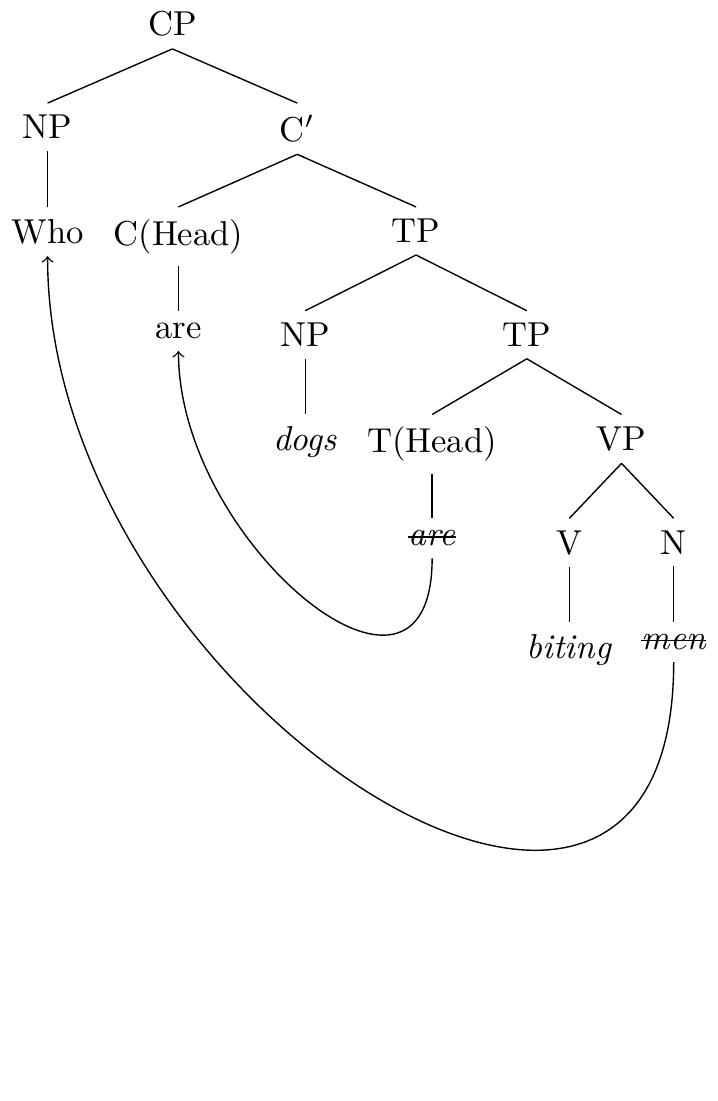}
\caption{Move operation applied to generate wh question}
\label{wh_q}
\end{figure}

Fig \ref{dependent_vs_independent} becomes important in understanding the transformation that leads us to sentences of the wh-form or the yes-no interrogative forms. Notice how the rule for generating these questions rely on the abstract labels given to the constituents (T(head) to C(head)). \textit{Therefore the Merge $+$ Move operation shows that for a generated sentence to be grammatical, the transformation must take as its domain} $\cup_{j=1}^{i} T^{*}_{j}$. 

\begin{definition}
A structure-based dependency $g_{sd}$ is one which relies on the information of the atomic elements that make up the input sequence, and the compositional structure that underlies the surface order. \\

$g_{sd} : g_{sd}(T^{*}_{i},T^{*}_{i-1},..,T^{*}_{2},s_{1}) = s_{2}$, where $s_{1},s_{2}\in S$.\\

On the other hand, a rigid or structure independent operation can be defined as $g_{si}$: \\

$g_{si} : g_{si}(s_{1}) = s_{2}$, where $s_{1},s_{2}\in S$ \\
\label{defintion_structure_dependence}
\end{definition}

A structure-dependent function differs from a rigid dependency in that they operate on sets which are non-isomorphic.

\begin{lemma}
$S \not\cong T^{*}_{i}$. In addition, given two sets $T^{*}_{i}$ and $T^{*}_{j}$ such that $i \neq j$, where the sets are defined from \ref{def_t_star}, $T^{*}_{i} \not\cong T^{*}_{j}$, that is, they are non isomorphic. 
\end{lemma}

\begin{proof}
An isomorphism is also referred to as one-to-one or bijective correspondence. This means that given two sets $X, Y$, with arrows sending elements of $X$ to elements of $Y$ (it is helpful in this instance to think of functions as arrows connecting elements of two sets), no two arrows from $X$ will hit the same element of $Y$, and every element of $Y$ will be in the image. Therefore, to show that two sets are non-isomorphic, it is enough to show that there exist two elements in $X$ such that both map to the same element $y\in Y$. It is trivially true from the formalisation described above that two or more elements of $S$ can have the same positional patterns of $L_{i}$ constituents - Take for example the sentence ``Big dogs bite men" and the sentence ``Big men bite dogs". Both have $L_{2}$ constituents ending at the time steps $(2,4)$, and therefore $S \not\cong T^{*}_{i}$.(See \ref{appendix} for detailed example)\\

To prove that $T^{*}_{i} \not\cong T^{*}_{j}$, we consider the diagram in \ref{proof_1}. It is enough to show that $T^{*}_{i} \not\cong T^{*}_{i-1}$, other cases can be proven analogously. We therefore replace $i$ and $j$ by $i-1$ and $i$ respectively, and assume the existence of two functions $g_{1}:T^{*}_{i-1}\to T^{*}_{i}$ and $g_{2}:T^{*}_{i}\to T^{*}_{i-1}$ and show how $T^{*}_{i} \not\cong T^{*}_{i-1}$ since $g_{1},g_{2}$ do not satisfy properties of a function. For isomorphism to hold, the diagram in \ref{proof_1} must commute. That is, the two paths - $f_{i-1}\circ \pi_{i-1} \circ g_{1} =  f_{i}\circ \pi_{i} \circ g_{2}$. Since we are considering operations on grammatically valid sequences to begin with, we consider $s_{1}, s_{2}\in S$ that are also grammatical. \\

Given that we work with binary branching trees, and the requirement that elements of $T^{*}_{i}$ contain \textbf{at least} one element that belongs to $T^{*}_{i-1}$ (from \ref{constituent_def}), we find that if $s_{1},s_{2}$ have an image in $T^{*}_{i}$, they must also have images in $T^{*}_{i-1}$. But that is not to say that both $s_{1}, s_{2}$ must have different images in both image sets. Let $s_{1}$ have an element $a\in T^{*}_{i}$ acting as the root node, and $t\in T^{*}_{i-1}$ as a part of its left branch. Similarly, let $s_{2}$ have another element $b \in T^{*}_{i}$ as its root node (to be distinguished from $s_{1}$ in the contents of its right branch), and have the same element $t\in T^{*}_{i-1}$ in its left branch. This means that when traversing the path from left to right - $f_{i-1}\circ \pi_{i-1} \circ g_{1}$, we have two elements in $S$ giving us identical images in $T^{*}_{i-1}$, which lead to two different images in $T^{*}_{i}$. This means that $g_{1}$ does not satisfy the basic requirement of a function, which is that one element have only one image. Therefore, the two sets are non-isomorphic.\\

\begin{figure}
\centering
\includegraphics{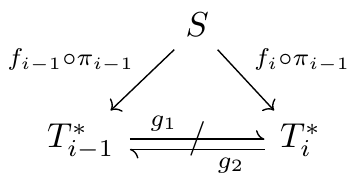}
\caption{Non-isomorphism of $T^{*}_{i}$s}
\label{proof_1}
\end{figure}
\end{proof}

A few important considerations are obtained as a result of imparting the above properties to a cognitive system \cite{martin_compositional_2020}. Firstly, compositional information processing systems should necessarily use a \textbf{time bound, incremental representation}, to reflect what we know from language. Also, from def \ref{defintion_structure_dependence}, in order to possess structure dependence, the bare minimum is representation. That is, the \textbf{explicit, recoverable presence of compositional elements}. \\

This chapter was meant to get a grip on essential linguistic concepts from a formal perspective. The development of modern NLP techniques sprang from the idea that strings and sequences in natural language have statistical properties, which is of course, not disputable. But the problem of teaching machines language will remain incomplete as long as the necessity for structure is not recognised. The properties and constraints we have seen in this chapter are partially fulfilled by DORA. We now see what properties DORA possesses that make it suitable to process language.

\label{part:body2}
\chapter{Implementation}
\label{chapter:body2}
To get an algorithmic understanding of how a few properties described in the formalisation can be simulated in a cognitive model, we look at DORA (Discovery of Relations by Analogy)\cite{doumas_theory_2008}. In these simulations, DORA is used as a classifier, which - given a set of natural language propositions, uses a Self Supervised Learning (SSL) routine to find mapping connections between elements of the propositions. These simulations were run to see the evolution of mapping strength over time, and to understand what kind of word representations are most favourable to learn the structural similarity between propositions. How such a learning routine changes the network configuration of DORA, and what it means for the formalisation, is the subject of discussion in chapter \ref{chapter:resultsanddiscussion}. 

\section{Introduction to DORA}

DORA is a cognitive model of relational learning and analogy. There are a few minimal constraints that need to be imposed on a system if it is to be capable of processing relationships. First-order logic defines a predicate as a function of zero or more variables that return Boolean values. A multiplace predicate can be recast as a collection of single-place predicates (one for each role of the relation), with functions for linking them. This necessitates using representational elements to correspond to the predicate and the object, and a way to combine them. \\

To maintain the notion of a \textit{variable}, DORA must be able to keep relational roles separate from the arguments, and also make the binding dynamic. That is, a mechanism of binding must be independent of the elements themselves. Keeping these requirements in mind, we take a look at the architecture DORA uses.(All the terms are described in the tables \ref{dora_components} and \ref{dora_inhibition_components})

\subsection{Architecture}

Propositions in DORA are represented across four layers, termed the P, RB, PO, and semantic units. We first look at how each proposition in DORA is encoded (ie., the micro-architecture), and then understand how the encoding of individual propositions relates to how an entire knowledge base is represented in DORA (ie., the macro-architecture).

\begin{figure}
    \centering
    \includegraphics{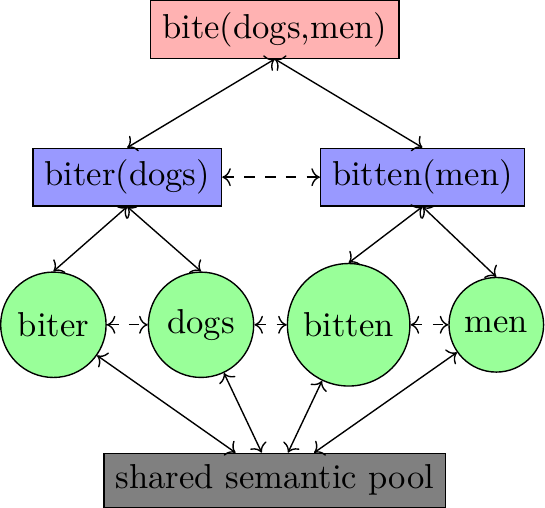}
    \caption{Representing the proposition - bite(dogs, men) in DORA. Topmost is the P layer, followed by RB layer, then PO layer, then the semantic layer. Bidirectional connections are indicated by vertically aligned arrows, horizontally aligned arrows are lateral inhibitory connections between units of the same layer.}
    \label{sample_prop}
\end{figure}

\subsubsection{Semantic layer}
The bottom-most layer is the semantic layer and it acts as the perceptual layer connecting the real world to DORA. This is done via representing entities as distributed features, which are commonly accessible to all the units in the layer above it. In the original instantiation of DORA and LISA \cite{doumas_theory_2008}\cite{hummel_distributed_nodate} this semantic layer encoded features such as visual invariants (height, colour, size, etc.), relational invariants (more-than, less-than, etc.), complex categorical features (animal, country, etc.). Some of these featural invariants, such as visual features, are claimed to be innate to human processing abilities.

\subsubsection{PO layer}
The PO layer codes for the individual predicates and their objects. The connections between the predicate, object units, and the semantic units are given by link weights joining them. This enables the passing of activation from the semantic layer to the PO layer (and vice versa) as a factor of the link weight. For example, the object `dogs' can have certain semantic features associated with it, such as (`size',`fur','bark') and link weights (0.4, 0.8, 0.95) The link weights take values in range (0,1) ). The predicate and object units are structurally identical, and are differently coloured in figure \ref{sample_prop} only to make the conceptual difference explicit. 

\subsubsection{RB layer}
Above the PO layer sits the RB or Role-binder layer, which encodes the bindings of the relational roles to their fillers. In \ref{sample_prop}, there are two RB units encoding the proposition bite(dogs,men), given by biter(dogs) and bitten(men). This utilises the idea that n-ary entities can be broken down into n binary-entities, where each binary unit is represented in the form of one RB unit. As it will be detailed in the following sections, an RB unit can also have a P unit as its child, thereby imbuing DORA with powerful recursive abilities. Every RB unit shares bidirectional excitatory connections with its PO units, and also to the P unit above it. 
\subsubsection{P layer}
The P layer is the topmost layer, and consists of full propositions. It is capable of binding n-array of RB units, but in keeping with binarity, we look at P units as conjunctively binding two RB units. The P units share bidirectional excitatory connections with their RB units. All units also laterally inhibit other units in their layer. That is, PO units inhibit other PO units, RB units inhibit other RBs, and P units inhibit other Ps. This is shown in the horizontal dashed lines in figure \ref{sample_prop}. 

\subsection{Memory Banks}
Propositions in DORA are members of three sets at any given point in time - the Driver, Recipient, and Long-Term Memory (LTM). The Driver is intended as an analogue of the working memory or the focus of attention, and the flow of all activities in DORA starts from it. Propositions in the Driver pass activation downwards, and into the semantic layer, which is shared between the Recipient and the LTM. The LTM is mainly a storage space for all propositions, and during retrieval, DORA moves propositions into the recipient for mapping. These patterns of activation in the semantic layer are what drive relational mapping and structure discovery, as elaborated in the following sections.

\begin{figure}[h]
    \centering
    \includegraphics[scale=0.7]{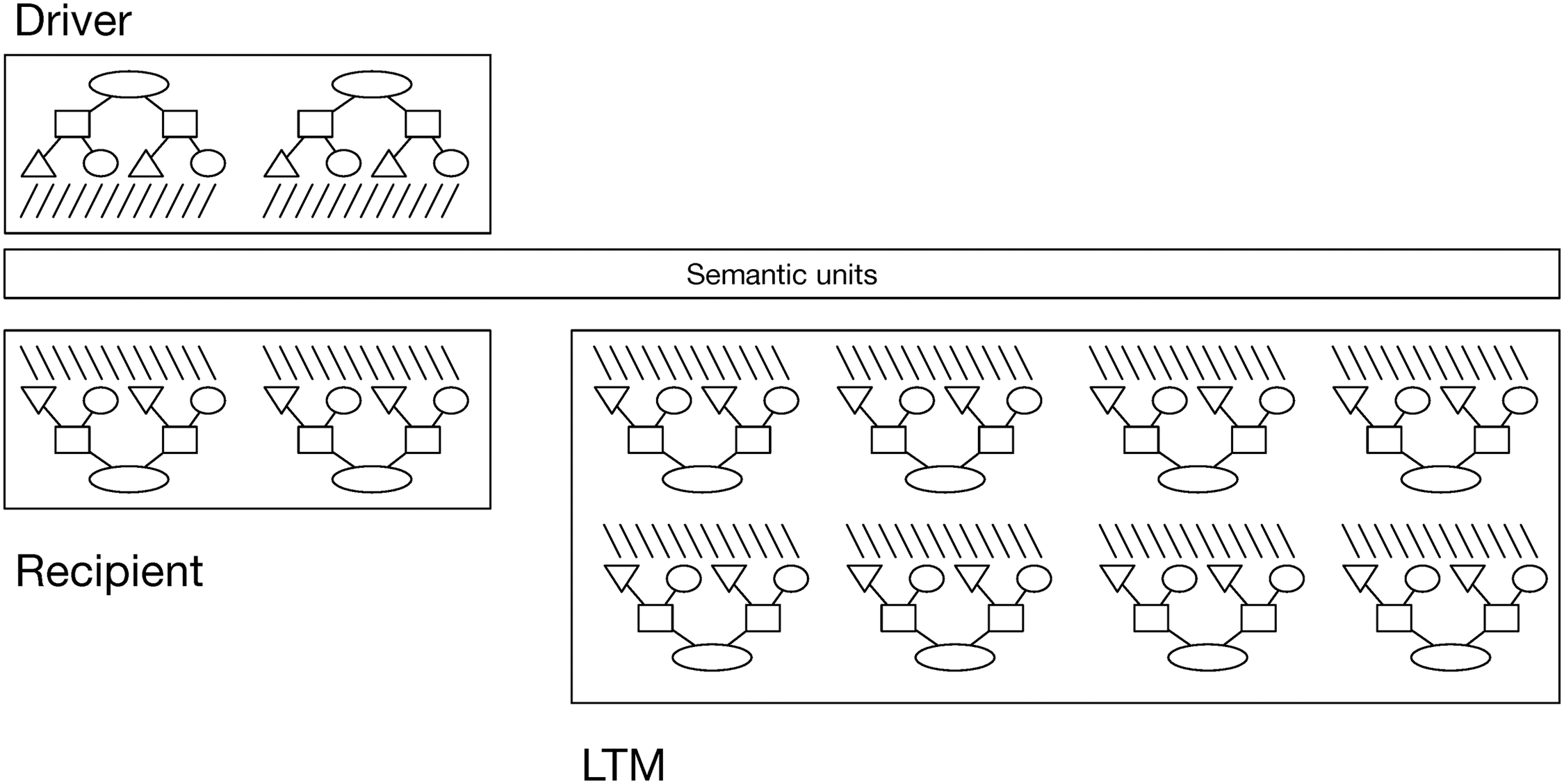}
    \caption{Banks of units in DORA containing represented propositions, from \cite{martin_mechanism_2017}}
    \label{DORA-banks}
\end{figure}

\subsection{Dynamic role-filler binding}
We have covered the architectural setup of DORA, and noted the minimum processing constraints. Now we see how DORA achieves this. DORA uses a generalised notion of firing synchrony as a way of keeping role-filler bindings distinct. That is, information is carried by when role-filler units fire rather than which units fire. When a proposition in the driver becomes active, it passes activation to its RB units which fire out of synchrony with each other. See \ref{illustration_dora_binding} for a detailed explanation of how DORA processes a proposition.\\

This interplay between synchronous and asynchronous firing is a result of the complementary excitatory and inhibitory connections that DORA possesses. Each RB and PO unit possesses an excitatory unit yoked to an inhibitory unit \cite{doumas_theory_2008}. The exciters accept inputs from the bidirectional connections to the layers above and below them, and passes a part of this to its inhibitor, and inhibitors at lower levels. As mentioned before, the individual units also share lateral inhibitory connections with other units of the same type. The result is that groups of oscillators that share inhibitory connections will tend to oscillate out of synchrony with each other, due to complementary push-pull mechanisms. Note that PO units obtain inhibitory inputs both laterally and from RB units above them, which means that a typical PO inhibitor obtains twice as much input compared to an RB inhibitor. Therefore, PO units oscillate with twice the frequency of an RB unit.\\ 

To distinguish boundaries between individual RBs and individual POs, DORA produces an inhibitory `refresh' signal when RBs (and analogously POs) are not active above a certain threshold. The resulting pattern for a proposition that uses two roles is then - role1 - refresh - filler1 - refresh - REFRESH - role2 - refresh - filler2 - refresh - REFRESH, where refresh is the PO refresh signal, and REFRESH is the RB refresh signal. This refresh signal is created by the local and global inhibitors as described below - 

\begin{longtable}{|p{2cm}|p{\linewidth-3cm}|}
\hline
Inhibitor type & Description \\[0.5ex] 
\hline\hline
Yoked inhibitor & Every unit (except semantic layer) is coupled with an inhibitor. The purpose of the PO and RB inhibitors is to establish the time sharing that carries role filler binding information.\\
Local/Global inhibitor & Serve to coordinate activity between driver and recipient. Activation of local inhibitor, $\tau_{L} = 0$ when any PO unit in driver has activation > 0.5. Otherwise, $\tau_{L} = 10$. During asynchronous time sharing, there is a period of time when Predicate has fired, and object isn't active. This is when local inhibitor becomes active, and acts as a local refresh signal for changes from Pred-object activation to be reflected in the recipient.The global inhibitor works similarly, but tracks RB unit changes. ie., $\tau_{G} = 0$ when RB activations > 0.5, $\tau_{G} = 10$ otherwise.\\
\hline
\caption{Inhibition in DORA}
\label{dora_inhibition_components}
\end{longtable}

Note that the most important use of asynchronous binding is that information is bound with time, and not the properties of the unit itself, which allows equivalent description, with dynamic binding.

At this point, it is useful to summarise all the definitions -

\begin{longtable}{|p{2cm}|p{\linewidth-3cm}|}
\hline
Term & Description \\[0.5ex]
\hline\hline
Semantic layer & Common to all PO units, where every unit represents a single dimension along which predicate/object is encoded.\\
Link weights & Connect PO units to the semantic layer. Activations are passed as a factor of link weight.\\
PO Layer & Units represent individual predicates and objects. Objects can contain higher-order propositions.\\
RB layer & Units represent the binding of role (predicate) and filler (object).\\
P layer & Units represent the overall proposition\\
Analog & All propositions belonging to a single story/environment, which reuse objects and predicates amongst their propositions.\\
Driver & Represents the Field of Attention of DORA. All activations start in the Driver, and flow into the Recipient or LTM.\\
Recipient & Contains the propositions which are retrieved from LTM to enable mapping.\\
Long term memory (LTM) & Contains all the propositions from all analogs.\\
Bidirectional connections & The links connecting units of two different layers are excitatory, serving to propagate activation through the network.\\
Lateral inhibition & A competitive mechanism, where units of the same type seek to reduce the activation of its neighbours.\\
Retrieval & The process of moving coactive units/analogs from LTM to Recipient as a precursor to Mapping/predication/schema induction etc. \cite{doumas_theory_2008}\\
Mapping & The process of discovering which elements in driver correspond to which elements in the Recipient.\\
Phase set & Set of mutually desynchronised RB units. Every phase set is set to run 3 times, updating mapping connections (from hypotheses) at the end of every set. For the current implementation, a phase set is the set of all units in the firing order.\\
Firing Order & To simulate DORA, the firing order of PO units (analogous to sequential presentation of words) in the Driver is set up before DORA begins its learning routine.\\
Mapping hypotheses & Uses Hebbian learning to build evidence for a connection between propositions in the driver to propositions in the recipient. \\
Mapping connections & Uses mapping hypotheses to update mapping connections that are committed to LTM.\\
\hline
\caption{Summary of components in DORA}
\label{dora_components}
\end{longtable}

\subsection{Flow of Control}

In this section, we look at the various steps in the self-supervised learning routine used by DORA to form mappings. Given many analogs in the LTM, DORA first picks out a random analog into the driver to begin its routine. Then it performs Retrieval, to place into the Recipient, the nearest remaining analog from the LTM. If retrieval is successful, it starts the mapping process. The routines which pass activation from driver to LTM/Recipient are coloured blue, retrieval is coloured orange, and mapping processes are coloured green.

\begin{figure}[h]
    \centering
    \includegraphics[scale=0.6]{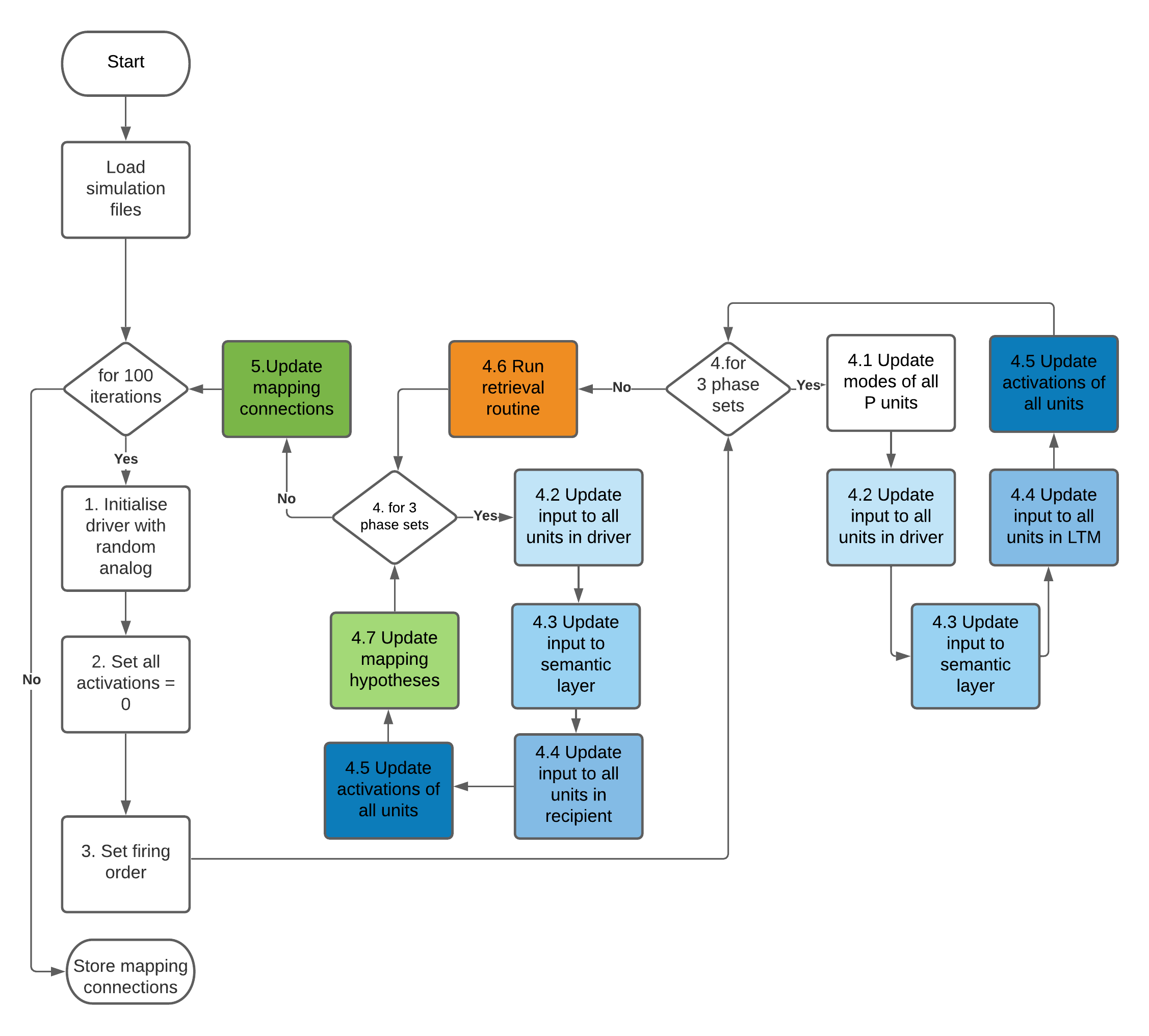}
    \caption{Flow of control in DORA}
    \label{flowchart_dora}
\end{figure}

Now we will look at a brief explanation of these individual functions. Readers are encouraged to go through \cite{doumas_theory_2008} to get an in-depth understanding. At the start of the simulations, DORA's LTM is initialised with the simulation files containing the propositions as detailed in section \ref{dataset_representation}.

\subsubsection{Step 1,2,3 : Initialise network}
In the first step, a randomly chosen analog (containing the Ps, RBs, and POs) is loaded into the Driver, and activations of all units in DORA is set to 0. A randomly chosen analog, together with all of its units is moved from the LTM to the Driver. To run simulations, the firing order is set to the sequence of POs in the driver. That is, given the contents of a driver, which contains $n$ propositions, the firing order contains the concatenated list of all words appearing in these $n$ propositions. For example, given only a proposition in the driver - ``big dogs bite cats'', the firing order is = PO(big), PO(dogs), PO(bite), PO(cats). The activations of all these PO units are clamped to 1, until the local inhibitor fires.

\subsubsection{Step 4.1 : Update P modes}

P units operate in three modes - Parent, Child, Neutral at any given point in time. That is, a P unit can be in child mode if it acts as a filler to an argument. The mode of a unit $i$ is $m_{i}$ and is given by :

\begin{center}
\[   
    m_{i} = 
         \begin{cases}
           \text{Parent(=1)} &\quad\text{$RB_{above}<RB_{below}$}\\
           \text{Child(=-1)} &\quad\text{$RB_{above}>RB_{below}$} \\
           \text{Neutral(=0)} &\quad\text{otherwise}
         \end{cases}
    \]
\end{center}

$RB_{above}$ is the summed input from RB units to which $P_{i}$ acts as an object. $RB_{below}$ is the summed input from $P_{i}$s downward connection. Mapping hypotheses are only set up between P units of the same type (shown in step 4.7).

\begin{definition}
Input to a unit $i$ is denoted $X_{bank,i}$ where $X \in $ \{P,RB,PO\}, $bank \in $ \{Driver, Recipient, Memory\}. Semantic units are denoted $SEM_{i}$. 
\end{definition}

\subsubsection{Step 4.2 : Update inputs to all units in driver}
\begin{center}
\[
P_{driver,i} = 
    \begin{cases}
        \sum_{j} a_{j} - \sum_{k} 3a_{k} &\quad\text{$m_{i} = 1 or 0$}\\
        \sum_{j} a_{j} - \sum_{k} a_{k} - \sum_{l} a_{l} - \sum_{m} 3a_{m} &\quad\text{$m_{i}=-1$}
    \end{cases}
    \]
\end{center}
P units update their inputs according to the equations shown above. In case of the parent mode, input is the sum of downward $RB_{j}$, and lateral connections $k$ in parent mode. In the child mode, input is the sum of upward $RB_{j}$, other P's in child mode ($a_{k}$), $l$ is all PO units not connected to same RB as $P_{i}$, $m$ is all PO units connected to same RB as $P_{i}$. 

\begin{center}
\[
RB_{driver,i} = \sum_{j} a_{j} + \sum_{k} a_{k} - \sum_{l} 3a_{l} - 10I_{i}
\]
\end{center}
Input to the $RB_{i}$ is given by the $P_{j}$ in parent mode, $k$ are PO units connected to $RB_{i}$, $l$ are other RB units in the driver, $I_{i}$ is the RB inhibitor yoked to it.
\begin{center}
\[
PO_{driver,i} = \sum_{j} a_{j}G + \sum_{k} a_{k} - \sum_{l} a_{l} - \sum_{m} 3a_{m} - \sum_{n} a_{n} - 10I_{i}
\]
\end{center}
Input to the $PO_{i}$ is given by the upward connected $RB_{j}$, and gain parameter (G=2 for predicate, =1  for object), $k$ are P units in child mode and not connected to same RB as $i$, $l$ is all PO units not connected to same RB as $i$, $m$ are PO units connected to same RB, $I_{i}$ is activation of yoked inhibitor of PO.\\

Remember how every unit is also yoked to its inhibitor. The input to such a yoked inhibitor at any point $t+1$ is given by -
\begin{center}
\[
Inhibitor_{i}^{t+1} = Inhibitor_{i}^{t} + \sum_{j} a_{j}w_{i,j}
\]
\end{center}
$t$ is the current time step, $j$ is the RB or PO unit yoked to inhibitor, $w_{i,j}$ is the weight between inhibitor $i$ and the unit it is yoked to. RB inhibitors are yoked only to their RBs, whereas PO inhibitors are yoked to both their own POs and RB units. As a result, PO inhibitors get twice as much input compared to RB inhibitors, and therefore fire twice as frequently.\\

\subsubsection{Step 4.3: Input to semantic layer}

\begin{center}
\[
SEM_{i} = \sum_{j\in PO_{D},PO_{R}} a_{j}w_{i,j}
\]
\end{center}

Input to a semantic unit is the sum of the product of inputs from PO units in the driver (and recipient) and the link weights between them.

\subsubsection{Step 4.4: Update input to recipient/LTM units}

At this point, it is important to distinguish between the various inhibitor units. \textbf{Yoked inhibitors}, $Inhibitor_{i}$ are coupled to PO and RB units. On the other hand, the \textbf{local and global inhibitors} serve to coordinate activity between the Driver and Recipient.

\begin{center}
\[
P_{LTM(or Recipient),i} = 
    \begin{cases}
    \sum_{j} a_{j} + M_{i} - \sum_{k} 3a_{k} - \tau_{G} &\quad\text{$m_{i} = 1, 0$}\\
    \sum_{j} a_{j} + M_{i} - \sum_{k} a_{k} - \sum_{l} a_{l} - \sum_{m} 3a_{m} - \tau_{G} &\quad\text{$m_{i} = -1$}
    \end{cases}
\]
\end{center}

The subscripts are indicative as in the Driver P input update in step 4.2.

\begin{center}
\[
M_{i} = \sum_{j} (3a_{j}w_{i,j} - Max(Map(i)) - Max(Map(j))) 
\]
\end{center}

The additional term $M_{i}$ corresponds to the mapping connections from unit $i$ to $j$ units in the driver it is connected to. These mapping connections pass activation directly to the recipient units from other units in the driver. $Max(Map(i))$ is the maximum value of the mapping connections originating in unit $i$. 

\begin{center}
\[
RB_{LTM(or Recipient),i} = \sum_{j} a_{j} + \sum_{k} a_{k} + \sum_{l} a_{l} + M_{i} - \sum_{m} 3a_{m} - \tau_{G}
\]
\end{center}
Input to RB units - $j$ is P units in parent mode to which $i$ is upwardly connected, $k$ is P units in child mode (downwardly connected), $l$ is PO units connected to $i$, $M_{i}$ is mapping input, $m$ is other RB units in recipient, $\tau_{G}$ is the global inhibitor.

\begin{center}
\[
PO_{LTM(or Recipient),i} = \sum_{j} a_{j} + SEM_{i} + M_{i} - \sum_{k} a_{k} + \sum_{l} a_{l} - \sum_{m} 3a_{m} - \sum_{n} a_{n} - \tau_{G} - \tau_{L}
\]
\end{center}
Input to the PO units - $j$ is the RB unit connected to $i$, $SEM_{i}$ is the semantic input to $i$, $M_{i}$ is mapping input, $k$ is all POs not connected to same RB, $l$ is all P units in child mode not connected to same RB, $m$ is PO units connected to same RB, $n$ is RB units in recipient to which unit $i$ is not connected.

\subsubsection{Step 4.5 : Update activations}
DORA uses a simple leaky integrator as shown in \ref{activation_equation}, with $\gamma = 0.3, \delta=0.1$. RB and PO units update activation as : 
\begin{equation}
\Delta a_{i} = \gamma n_{i}(1.1 - a_{i}) - \delta a_{i}]_{0}^{1}
\label{activation_equation}
\end{equation}
Where $\Delta a_{i}$ is the change in activation of the unit,
$n_{i}$ Net input to the unit, as calculated in the previous steps, $\gamma$ is the growth rate and $\delta$ is the decay rate. The activation is also clipped to lie in the range - [0,1]. The net input at any point in time is dependent on the type of unit, and its presence in the driver/recipient. 

\subsubsection{Step 4.6,4.7,5 : Retrieval and Mapping}
\label{section_retrieval_mapping}

A preparatory step in DORA is the retrieval of analogs (a set of propositions representing an environment) from the Long term Memory (LTM) into the recipient for further operation. The assumption behind this step is that human cognitive processing is limited by working memory, which is why there is a separate step to `chunk' or move propositions from LTM to Recipient. The activation and inputs of units in the LTM are updated the same as units in Recipient. Retrieval in DORA can be described as guided pattern recognition. A proposition in the driver passes activation to the semantic layer, and these patterns of activation lead to co-active excitation in the LTM. After all the propositions in the Driver have fired, DORA retrieves analogs probabilistically using the Luce choice algorithm - 
\begin{equation}
    p_{i} = \frac{R_{i}}{\sum_{j} R_{j}}
\label{luce_equation}
\end{equation}

Where $p_{i}$ is the probability that the analog $i$ will be picked, given the activation $R_{i}$ and the sum of all other analogs in the LTM in $\sum_{j} R_{j}$. A if $p_{i} > r$, where $r$ is a random number, $r\in (0,1)$.When one analog retrieved, all the RB and PO units that belong to the analog are added to the Recipient.\\

Mapping is the subject of our primary attention. It is the process of discovering which representational element (including POs, RBs, and Ps) in the driver that DORA is currently looking at (in the driver), matches with the elements in the Recipient. \textit{It is the algorithmic equivalent of finding the membership of the current proposition with respect to the equivalence classes formed on the set of propositions in DORA.} \\

Mapping hypotheses are generated for every unit in the Driver, and they represent possible connections to every other unit of the same type in the Recipient. These mapping hypotheses accumulate evidence at each point in time, for a mapping connection between two units using a simple Hebbian learning rule :

\begin{equation}
    \Delta h_{i,j}^{t} = a_{j}^{t}a_{i}^{t},
\label{hebbian_learning_equation}
\end{equation}

where $\Delta h_{i,j}^{t}$ is the mapping hypothesis at time $t$, and $a_{i},a_{j}$ are the activations of the units under consideration. The mapping weights between these units are updated at the end of every phase set. Since these mapping connections are also excitatory, they serve to constrain further discoverable relations by passing activations directly to previously discovered mappings. This is foundational to many properties of human analogical processing \cite{doumas_theory_2008}\cite{hummel_distributed_nodate}. \\

From the formalisation, we recognise that any mapping between units must take into account the relative placement in time of its child units! With equation \ref{hebbian_learning_equation_modified}, we constrain mapping by using Pearson's correlation of RB units' children's activation. That is, not only RB's activation, but the correlation of the time patterns of its children's activation is used as a feature. 

\begin{equation}
\Delta h_{i,j}^{t} = a_{j}^{t}a_{i}^{t} + corr(a_{j,children}^{t}a_{i,children}^{t}) 
\label{hebbian_learning_equation_modified}
\end{equation}

\section{Dataset and representation}\label{dataset_representation}

In this thesis, DORA's mapping functionality is leveraged to form connections between natural language sentences across different analogs (see \cite{karthikeyakaushik_karthikeyakaushikdora_2020} for code). The problem of learning predicates for natural language sequences is an open question and beyond the scope of this thesis. What we instead investigate, is the representational constraints enforced by the formal method described in the preceding chapter - \textit{Cognitive systems capable of processing language must display incremental, structure-dependent functions.}\\

To understand these functions, we look at a synthesised dataset, with 4 analogs, consisting of 8 propositions in each analog. As mentioned before, one analog is a situation or environment. Therefore every analog consists of various objects and actions representative of that environment. For example, $analog_{1}$ represents a situation involving dogs and cats, and the actions relating them or other objects in their shared space. A few sentences in this analog and their representation in DORA is shown in \ref{analog_1}\\

\begin{figure}
\begin{tikzpicture}
\node(img) at (3,0)
    {\includegraphics[scale=0.7]{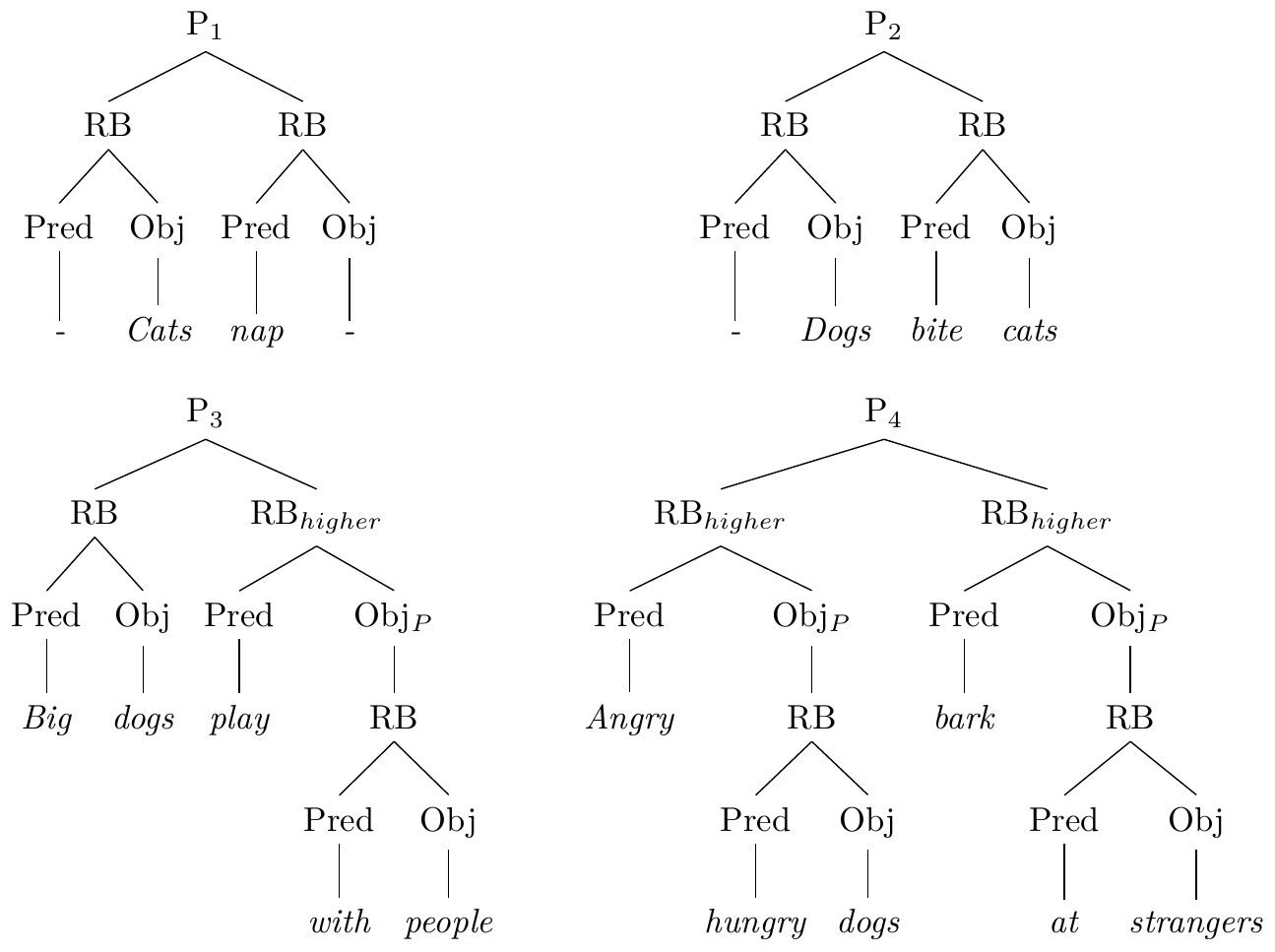}};
\node[draw,text width=\linewidth/2,scale=0.7] (txt) at (-5,0){$P_{i}$ : Proposition unit \\
$RB,RB_{higher}$ : RB unit, RB unit containing proposition\\
$Pred,Obj,Obj_{P}$ : Predicate and Object unit, Object unit containing higher order proposition\\
$`-'$ : Empty predicate};
\end{tikzpicture}
\caption{Examples of propositions in a single analog}
\label{analog_1}
\end{figure}

Notice that in some cases, we have an empty predicate. If all sequences are encoded in first-order logic, then the structure looks very different from \ref{analog_1}. For example, the sentence `Cats nap' should be encoded as nap(cats), where the predicate `nap' acts on `cats'. While this is true of propositions encoded in first-order logic, we cannot do the same due to constraints from a linguistic perspective. That is, `Cats' acts as the subject of the proposition where `nap' is the intransitive verb (verbs which do not take direct objects) the subject employs. In this way, we can justify the structure shown. Refer [..glossary..] for the set of all propositions used. \\

\subsubsection{Word2vec}

DORA uses the Semantic layer as a bridge to connect feature invariants between the LTM, Recipient, and the Driver banks. The question on how to represent words in DORA for this task was an important one and is also one of the criticisms that DORA faces. Previous work involving DORA or LISA have focused on interpretable and often hardcoded topical features. For example, in a typical setup, the features of a `Lion' would reflect its characteristics, such as size, voice, aggression, etc., It is certainly true that human concept formation involves this kind of dimensional encoding \cite{hummel_distributed_nodate}, but how ML techniques address this problem is an area of active research.\\

To overcome this constraint, we use word2vec representations in the semantic layer. We will now take a brief look at how these embeddings are obtained. \\

Modern language models such as RNNs (Recurrent Neural Networks) are very good at predicting what the next word is, given an incomplete sequence of words. One very important basis for this is the improved quality is embedded word representations\cite{alammar_illustrated_nodate}. Word embeddings are a class of methods where individual words are represented as real-valued vectors. Word2vec is a conceptual generalisation of Gottlob Frege's context principle - ``never ask for the meaning of a word in isolation, but only in the context of a proposition"\cite{zalta_gottlob_2020} . It is a generalisation because word embeddings are very good at describing the lexical and semantic content of words based on the corpus of data the language model has been trained on. There exist primarily two approaches to obtaining word embeddings - Continuous Bag of Words (CBOW) and Skip-gram. \cite{mikolov_distributed_2013}\cite{mikolov_efficient_2013}\\

The CBOW model generates embeddings by predicting the current word using a given context (past and future), while the Skip-gram method learns to predict the context using a given word. The CBOW is considered generally better only for frequent words, whereas Skip-gram is computationally more expensive, but does better on a sparser corpus.\\

\begin{figure}[h]
    \centering
    \includegraphics[scale=.7]{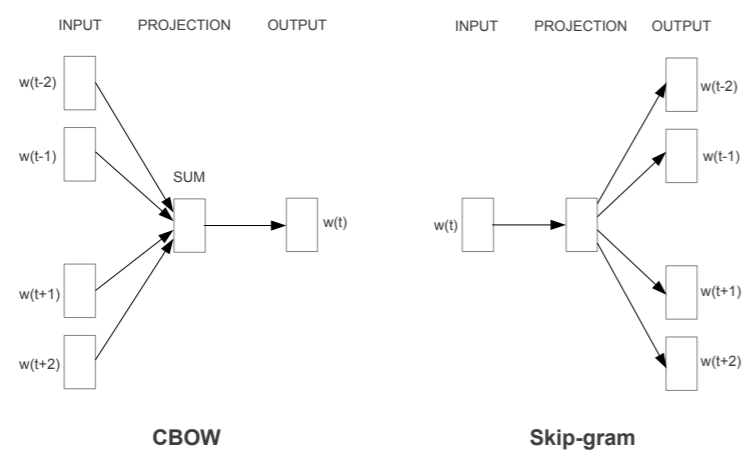}
    \caption{Comparison of cbow and sg, taken from \cite{mikolov_distributed_2013}}
    \label{cbow_sg}
\end{figure}

We run tests on two word embedding models, which are pretrained on large corpora of data - Google word2vec, which has been trained on 300 billion words from the Google news corpus\cite{noauthor_google_nodate}. The WordGCN model, which is trained on the Wikipedia corpus, uses Graph Convolutional Networks to combine syntactic context for improving word representations without increasing vocabulary size\cite{vashishth_incorporating_2019}. Both of these pretrained models encode words as 300-dimensional vectors. \\

There are a few more issues to be resolved before encoding the embeddings into DORA. Considering DORA's computationally intensive architecture and the small dataset, it becomes important to reduce the memory requirement for representing words. We therefore use Principal Component Analysis (PCA) to reduce the dimension of the word embeddings. PCA is a method for obtaining maximally uncorrelated variables from the given datapoint$\times$variable matrix. In our case, the words are datapoints and embedding dimensions are variables. We first define $M\in \mathbb{R}^{n\times k} $, as containing $n, \mathbb{R}^{k}$ vectors arranged into a matrix. Therefore, the task of finding embedding dimensions that preserve maximal variance reduces to finding the eigenvalues and eigenvectors of the correlation or covariance matrix of $n$ datapoints \cite{helwig_principal_nodate}.\\

Given $M$, we first calculate the correlation matrix of the $n$ datapoints treating every dimension $k$ as a variable. Given $x_{i,j}\in M$, we define the correlation matrix as $R\in \mathbb{R}^{k\times k}$, where $\overline{x_{i}}$ is the mean of all datapoints in a certain dimension, $r_{l,m}\in R$ is the Pearson correlation between two dimensions $l,m$:
\[
r_{l,m} =  \frac{s_{lm}}{s_{l}s_{m}} = \frac{\sum_{i=1}^{n}(x_{il} - \overline{x_{l}})(x_{im}-\overline{x_{m}})}{\sqrt{\sum_{i=1}^{n}(x_{il} - \overline{x_{l}})^{2}}\sqrt{\sum_{i=1}^{n}(x_{im}-\overline{x_{m}})^{2}}}
\]

Given $R$, we calculate its eigenvalues - $\lambda_{1}\geq \lambda_{2}\geq ... \geq \lambda_{k}$, and the corresponding eigenvectors - $e_{1}, e_{2} ... e_{k}$. At this point, we choose the top $p$ eigenvectors which can - 
\begin{enumerate}
    \item Give the simplest possible interpretation of the data with the lowest value of $p$. A related restriction is the memory constraint DORA faces.
    \item The proportion of variation which the $p$ eigenvectors explain should be as large as possible.
\end{enumerate} 

Keeping this in mind, we choose $p=10$. Now we arrange the top $p$ eigenvectors into a projection matrix $W\in \mathbb{R}^{k,p} (k=300,p=10)$, and get the projection of each word embedding $y\in \mathbb{R}^{10}$:
\[
y = W^{'}x \quad\text{where $W^{'}$ is the transpose of $W$}
\]

The semantic units accept only excitatory connections from POs, which means the link weights cannot be negative. Therefore, a max norm is first applied on every (reduced) dimension to make all values lie in the (0,1) range. Then, every word embedding is $L_{2}$ normalised to make the magnitude of every embedding equal to 1. \\

To visualise how words are related in this embedding space, we use Pearson's correlation to calculate the similarity between words. We see the results of the embeddings obtained on $analog_{1}$ in \ref{word_embeddings}. The words have been arranged by syntactic category - Nouns (strangers, cats, etc.,), verbs (purr, training, etc.,), adjectives(big, angry etc., ), etc., to improve any inference that can be drawn from the visualisation. We see that the simple correlation of word embedding shows considerable clustering around syntactic category boundaries both in the Google300b pretrained vectors, and the WordGCN syntactic vectors. This clustering is more pronounced in the WordGCN embeddings (see the lighter colours chunked in the diagonal bands). This means that the word representations have explicit syntactic markers, which as we shall see, helps DORA form better quality mappings. We compare the results obtained on the two models in chapter \ref{part:resultsanddiscussion}.

\begin{figure}
\begin{center}

\subfigure[Google300b]{
    \label{fig1}
    \includegraphics[scale=0.45]{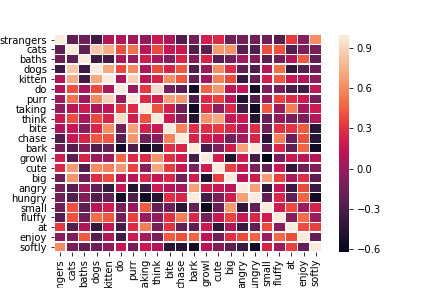}}
\subfigure[WordGCN]{
    \label{fig1}
    \includegraphics[scale=0.45]{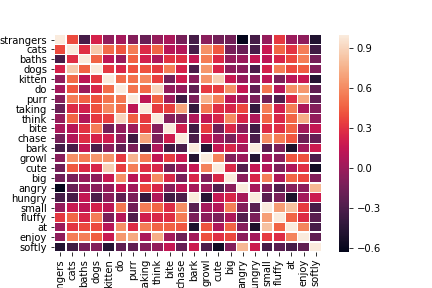}}
\end{center}
\caption{Visualised correlation matrices of word embeddings}
\label{word_embeddings}
\end{figure}

\label{part:resultsanddiscussion}
\chapter{Results and Discussion}
\label{chapter:resultsanddiscussion}

In the first section \ref{section-results}, we investigate DORA's learning, and results of simulations, and connect them to the requirements from the formalisation. Later, we address DORA's limitations, and discuss open questions and paths for the future. 

\section{Results}
\label{section-results}

\subsubsection{Learning}
DORA starts as a blank slate, with all propositions encoded in the LTM. As time progresses, DORA fills its mapping connections by the self-supervised learning algorithm described in \ref{part:body2}. In a single simulation, DORA's retrieval and mapping routines were run for 100 iterations and the network state examined. 10 such simulation runs were conducted for both embedding strategies, and the network evaluated using precision as metric (explained below). The baseline model equally distributes mappings from a given sentence to all sentences. That is, the baseline prediction is given by $b_{i,j}\in B^{n\times n}, n = \text{number of propositions}$:
\begin{equation}
    b_{i,j} = 1.0/n
\end{equation}

To compare the quality of mappings DORA generates, we collect all of its mapping connections into an adjacency matrix $M_{p}^{n\times n}$, where every $m_{i,j}\in M_{p}$ is the mapping from sentence $i$ to $j$. The true mapping matrix is given by $M_{t}^{n\times n}$, where $m_{i,j}\in M_{t}, m_{i,j} = 1 \iff i\neq j$ and $i$ has the same structure as $j$. Precision is the ratio of sum of correct predictions to the total predictions. It is given by the formula :
\begin{equation}
\label{fscore_formula}
\begin{gathered}
    precision = sum(TP)/(sum(TP) + sum(FP)) \\
\end{gathered}
\end{equation}
Remember that we are looking at the predicted adjacency matrix $M_{p}$, and the true adjacency matrix $M_{t}$(square symmetric). That is a binary, multiclass classification, where the prediction lies in the range [0,1]. Here TP (True Positives) is the sum of all elements of $M_{p}$ where $M_{t} = 1$ (Since there exist only positive mapping connections). Therefore, the TP is the sum of all values in $M_{p}$ masked by $M_{t}$. FP (False Positives) are the sum of those values in $M_{p}$ where there should be no mapping connection ie., values where $\neg M_{t}=1$. 

\begin{figure}[h]
    \centering
    \includegraphics{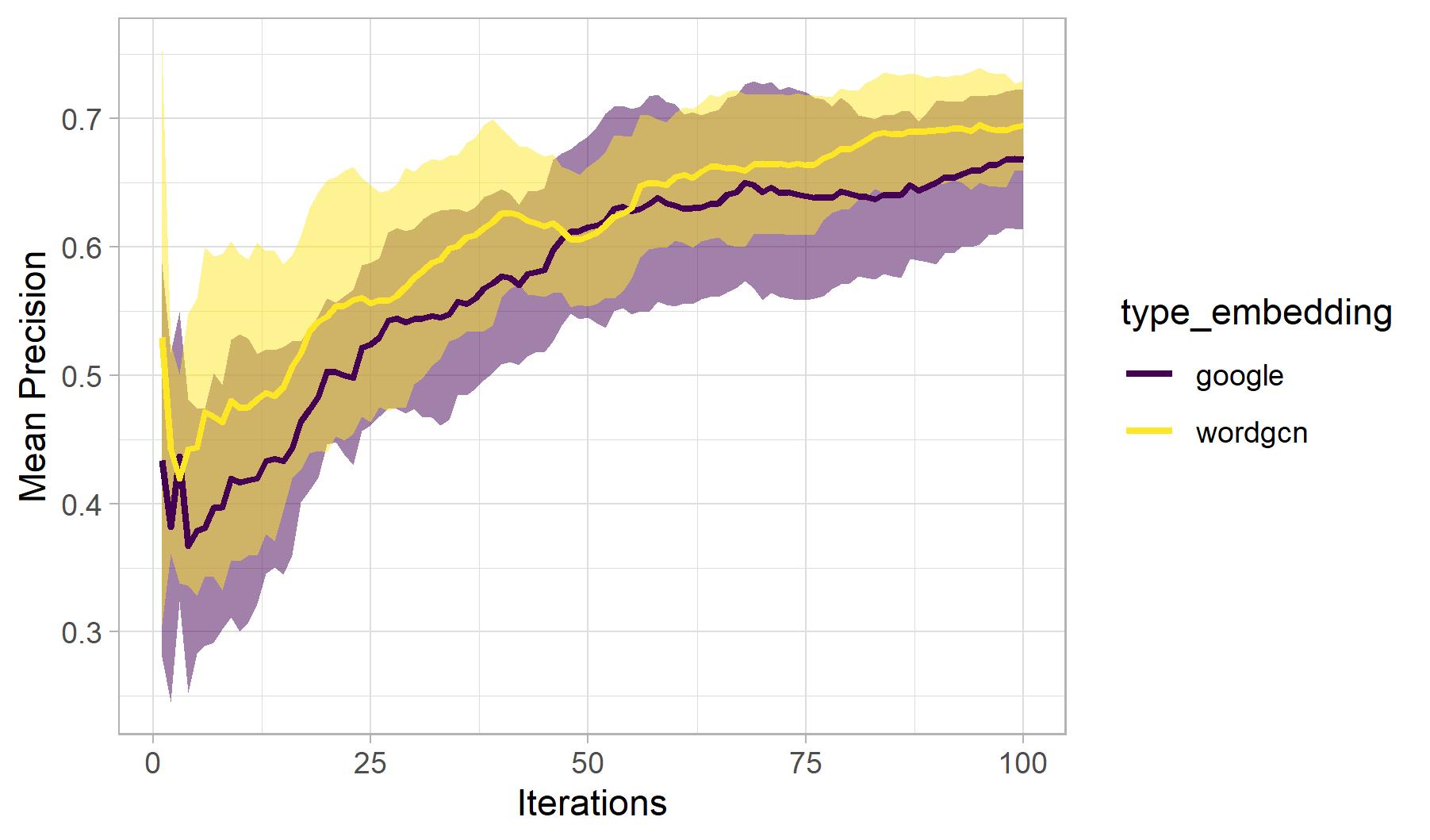}
    \caption{DORA's learning trajectory}
    \label{learning_comparison}
\end{figure}

We see in figure \ref{learning_comparison} how DORA's learning routine performs over iterations. Initially, DORA creates mapping connections that are noisy, resulting in a dip in the precision, after which it starts to make better predictions which also forms a basis for further constraining future mappings.

\begin{figure}[h]
    \centering
    \includegraphics{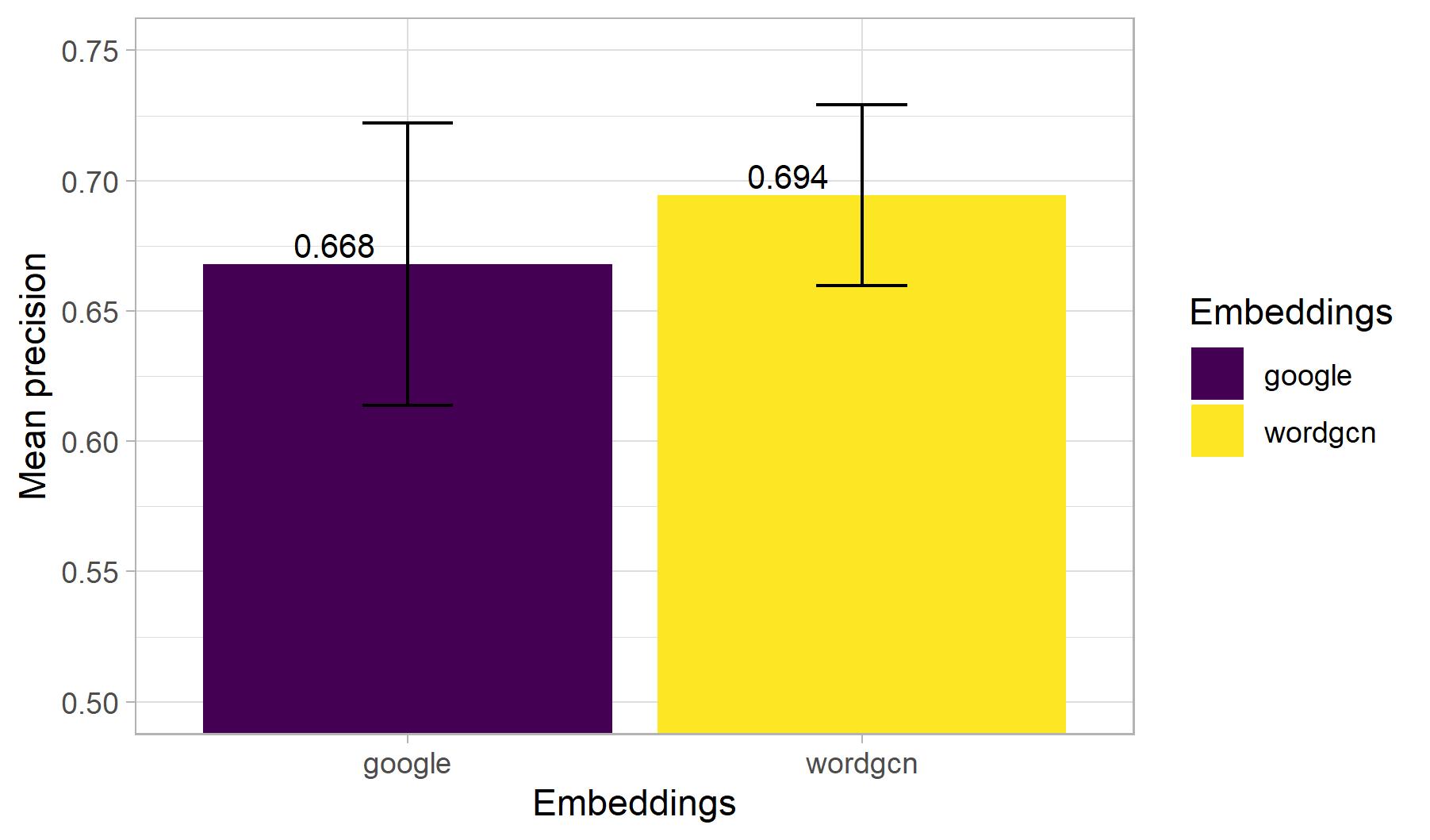}
    \caption{Comparison of embedding choices}
    \label{fscore_comparison}
\end{figure}

Recall that one of the aims was to find out if the quality of word embeddings and implicit markers in these embeddings help DORA to form better mappings. In \ref{fscore_comparison}, we first collect the precision at the end of 100 iterations for all repetitions. These scores are then averaged and plotted along with variances in \ref{fscore_comparison}(Baseline F score - 0.04). The two sample t-test is used to determine if two population means are equal \cite{noauthor_1353_nodate}. Or in other words, if one method is superior to the other. The WordGCN model was found to be significantly better than the google word2vec model (p<.01). Therefore, we conclude that the discovery of underlying structure is aided by the lexical properties of words made explicit in vector representations.

\section{Discussion}
\label{section-discussion}

DORA, originally intended as a model of analogical reasoning, has many properties that make it suitable as a model of language.

\subsection{Reusability of units}
DORA represents objects (in this case words) as explicit units that have distributed semantic representations. These units are bound into higher units, and are not specific to a single proposition. That is, in ``Dogs bite men'', the RB unit encoding ``bite men'' is not limited to only this proposition. It can be further used in any other proposition that contains ``bite men''. This fact is important because other Natural language processing models do not have a way to decompose compositional products like ``bite men'' into `bite' and `men'\cite{martin_tensors_2020}. The network state in these models is a black box from which these intermediate states cannot be identified. Which is of crucial importance given what we have seen from chapter \ref{chapter:body1}. That is, a structure-dependent operation takes in its scope, the explicit representation of compositional elements. 

\subsection{Mapping of compositional structures}

DORA's biggest strength is in the retrievable compositions it encodes, and the possibility of transformation on these structures. This is a direct implementation of the recursive function application on a sentence as in \ref{def_f}.\\

Take for example the sentence ``Majestic male lions roar in anger". In DORA, this sentence is represented internally as $P = (RB_{1}, RB_{2})$, where $RB_{1}$ contains - $Majestic_{p}(P_{child})$, $P_{child} = male(lions)$; $RB_{2}$ contains $roar_{p}(P_{child})$, $P_{child} = in(anger)$. Such a setup becomes increasingly complex when considering past tenses, compound words, etc., This particular structure can be justified since `majestic', which is an adjective, along with `male' act as predicates on the object `lions'. Similarly, `roar' acts on the prepositional phrase `in anger'. Now consider these words (with $P$ present in Driver) presented to DORA. Word representations are activated in the semantic layer, by clamping the specific PO unit's activation to 1. DORA then propagates this activation both upwards into the Driver, and downwards in the Recipient through the semantic layer \ref{DORA-banks}. Therefore, as time progresses, and until the local PO inhibitor fires, DORA passes activation to units further in the hierarchy. At every instant of time, each mapping hypothesis is updated by coactivation of the respective units, and correlations of children \ref{hebbian_learning_equation_modified}. \\

The steps of passing activation (in steps 4.2 through 4.5) correspond to the recursive function definition as in chapter \ref{chapter:body1}, \ref{def_f}. Information from intermediate units propagates upwards and inhibits other units of the same type. The activation from a lower RB/PO unit is reapplied as input to higher RB units. Hebbian learning \ref{hebbian_learning_equation} creates the mapping connections which is the implementation of the canonical inclusion function $i_{i}$. We see what direct activation through mapping units does in \ref{coactivation}. \\

To show how DORA's mapping constrains activations, we compare the network activation profile across time in \ref{coactivation}. Mapping allows for direct lateral transfer of activation between units ie., Structurally similar units become coactive, and proportionally inhibit all other units of the same type. On every iteration in \ref{coactivation}, the activations of all units are first set to zero. All the sentences in the memory are made active, one word at a time. At the end of every iteration, we collect the activations of all units and group them into their types. As we can see, when DORA's mapping connections get better, the inhibitory activity also increases, thereby making fewer units active. The way to interpret this result, is to recognise the increasing efficiency, recognition of structure brings about. See \cite{martin_mechanism_2017} for other interesting phenomena when processing linguistic stimuli with DORA. 

\begin{figure}
    \centering
    \includegraphics{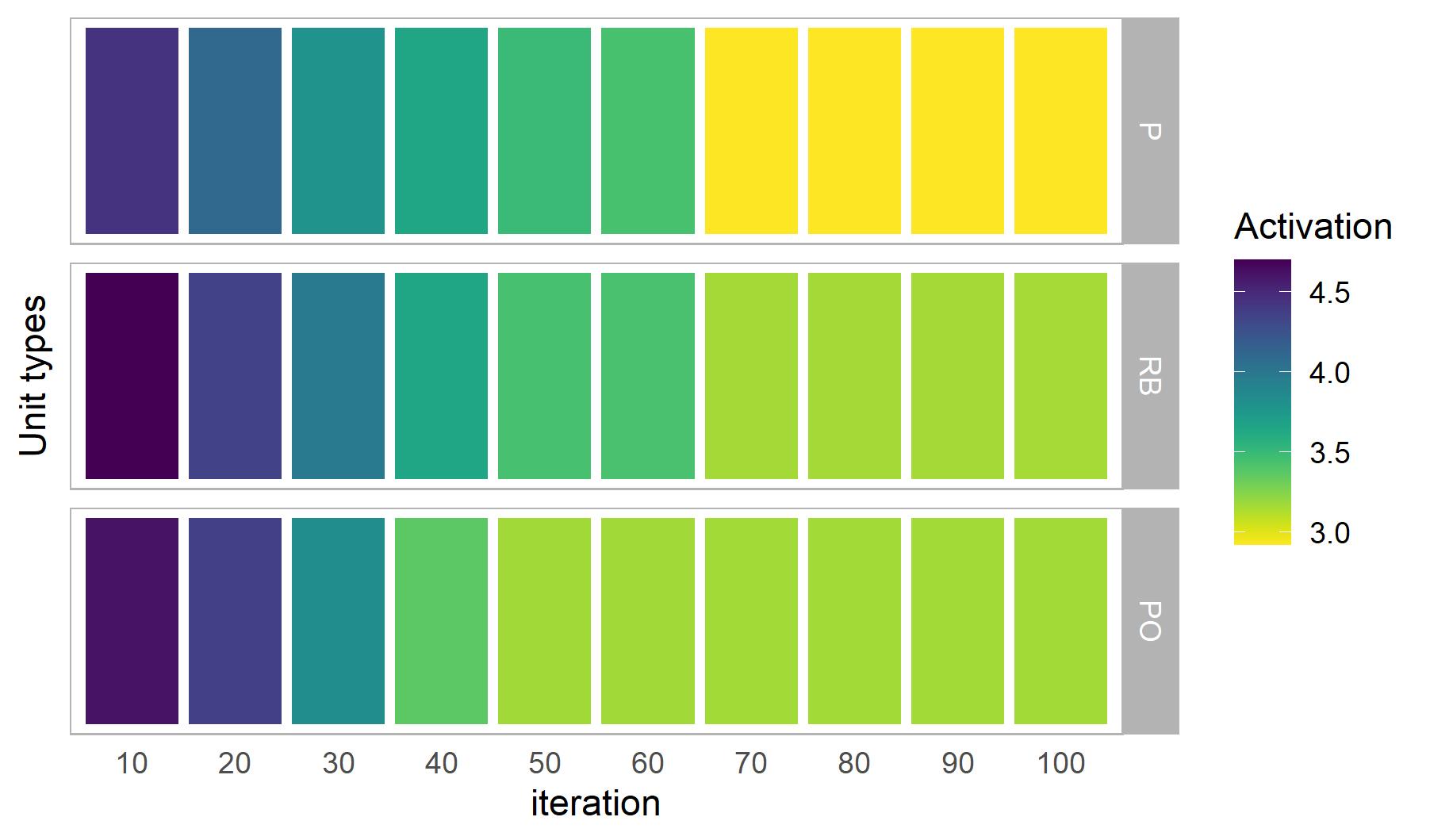}
    \caption{Coactivation leading to reduced summed activation across iterations}
    \label{coactivation}
\end{figure}

\subsubsection{Deep Neural Networks}

In \cite{martin_tensors_2020}, Martin and Doumas argue against tensor product based systems as models of cognitive processing. Consider the difference between structure-based and rigid dependencies as shown in \ref{struct_dependence_1} . The crucial difference between a truly compositional system, and one that approximates the behavior of a compositional system, is in the existence of explicit structural information in its domain. A model that is optimised to reflect the behavioural nature of human language use (translation, description, etc.,) need not necessarily have structure-dependent processing as in \ref{def_f}, but a cognitive model should certainly possess those ingredients. 

\subsection{Grammar induction}

We have seen in the formalisation, the Pushout created by the equivalence relation generated on $L_{i}$. The method of inducing an independent grammar or syntax can only begin once the system has observed the compositional nature of atomic elements. Although DORA possesses a schema induction functionality \cite{doumas_theory_2008}, it is not yet capable of maintaining an independent set of grammatical structures. Top-down imposition of structure on sequences is important in how humans perceive natural language \cite{martin_modelling_2020}. A readily adaptable solution exists in the form of mapping connections that DORA learns. An independent grammar is simply the set of mapping connections which are not bound to their units. 

\subsection{Evidence from neuroscience}

The MEG study in \cite{ding_cortical_2016} showed neural activity that correlated with hierarchically variable linguistic stimuli (syllables, phrases, and sentences). Participants were presented synthesised monosyllabic sequences in Chinese and English, at the rate of four words per second (4Hz). The structure of these sequences was manipulated by presenting - meaningful four-word sentences (``Dry fur rubs skin'', condition 1), two - two-word phrases (``Fat rats new plans'', condition 2), or random four-word sequences (``walk egg nine house'', condition 3). The power spectrum of the neuro-oscillatory output showed power increases at 1Hz, 2Hz, and 4Hz for the three conditions as shown in \ref{ding}. This signal has been attributed to the levels of syntactic structures that are latent in the stimuli. 

\begin{figure}[h]
    \centering
    \includegraphics[scale=0.4]{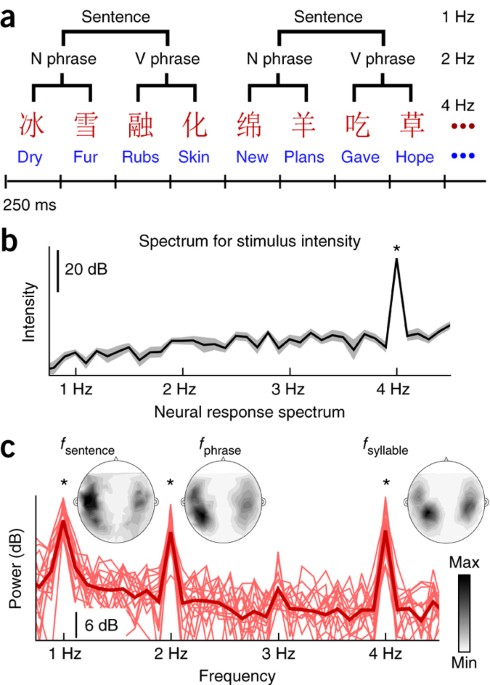}
    \caption{Taken from \cite{ding_cortical_2016}. \textbf{a} is the visualisation of the latent structure in condition 1, whose audio frequency spectrum is shown in \textbf{b}. \textbf{c} shows the neural response for the three conditions.}
    \label{ding}
\end{figure}

The formalisation in \ref{chapter:body1} makes room for this observation. In \ref{chapter:body1} and in \ref{detailed_1}, we see how every sequence can be mapped to $\mathbb{P}(L_{i})$. In this case, setting $i=2,3$ and with four words per second, we should see one element in $p^{*}_{3}$ appearing as an image on the application of $f_{3}(s)$, when $s$ is a sequence in the first condition, corresponding to the 1Hz signal. No such element exists when $s$ is two phrases strung together (second condition), but two $p^{*}_{2}$ images are seen for the phrase condition - corresponding to the 2Hz signal. As for the third condition, the words do not group together at all, therefore the frequency response of the stimuli is reflected in the brain activity. See \cite{martin_mechanism_2017} for details on how DORA displays similar activity when presented with the same stimuli. 

\subsection{Limtations of DORA}
DORA suffers from many of the problems faced by classical symbolic models. Here we discuss the issues DORA faces, and possible paths ahead - 
\begin{enumerate}
    \item Propagating activations through the network is a time consuming affair, with every unit updated individually, and not in parallel. Such a system can benefit tremendously by a matrix update approach. Such a method will also reduce unnecessary information made redundant by repeated creation. For example, every PO/RB/P unit (instantiated as an object in an Object-Oriented Programming paradigm) has a separate attribute inhibitor threshold, although every unit of a single kind has the same threshold. This redundancy can be overcome by taking an array-based approach. 
    \item A literature survey did not reveal pretrained topical word embedding models. That is, embeddings grounded in semantically interpretable dimensions. DORA's operations would benefit greatly from such an encoding, making compositions more `meaningful'.
    \item Occasionally, DORA makes errors in mapping. These incorrect mappings act as positive feedback which further constrains future activations. The SSL algorithm needs a more robust learning routine to overcome this issue.
    \item From our formalisation, we see that grammar is simply a restriction imposed on possible compositional structures that a linear sequence affords. With this in mind, DORA's SSL algorithm must be tweaked so that it becomes a Partial SSL method. That is, the separation of grammar can only take place with valid/invalid labels for compositional structures.  
    \item DORA is a complicated model, and understanding its processing is not particularly straightforward. More can be done to bring out a stable version which the open-source community can work on.
    \item Currently, the formalisation assumes a purely bottom-up, perceptual system that builds evidence for membership in a particular equivalence class. The recursive computation of merge products as described, is costly. Therefore $\mathbb{H}$ takes on added significance because of its predictive capacity. That is to say, whatever elements of $\mathbb{H}$ are known, are used as bootstrapping to constrain any further perceptual learning.
    \item DORA is at present incapable of learning predicates in natural language sentences. This is, of course, a non-trivial problem, and an area of active research. Consider why this is so problematic - Presented with a sequence of words, the model must be able to place every word at some level in the (unknown a-priori) hierarchical representation. This conclusion need not point to regression towards fully symbolic systems. A possible alternative is a statistically modified DORA, with the ability to maintain `active' representations of symbols instead of passive, fixed ones. Notice also that we do not know the length of the sequence (nor its contents) before processing. The model must rely on the memory of what it has previously seen, to determine the current structure. This suggests a direct inspiration from Recurrent Neural Networks which do a very good job of representing states of linear processing. Using state information to actively build structure is therefore a promising path going forward.
\end{enumerate}

\part*{Appendix}
\addcontentsline{toc}{part}{Appendix}

\appendix 

\chapter{Detailed Descriptions}
\label{chapter:DetailedDescriptions}\label{appendix}
\section{Overall processing}\label{overall_proc}
The intention behind \ref{detailed_1}, is to look in detail at two sentences processed by the formal model described in chapter \ref{chapter:body1}:
\begin{enumerate}
    \item \label{sent_1}Big dogs bite men
    \item \label{sent_2}Men say men bite men
\end{enumerate}

We see the two sentences in $S$ at the topmost part of \ref{detailed_1}. The two sentences and their corresponding images are distinguished by the colours of their bounding boxes. \\

From these two sentences, $f_{2}$ maps s\ref{sent_1}, s\ref{sent_2} to the elements $a,b \in \mathbb{P}(L_{2}) \times \{0,1\}$ , where $a=((big dogs,bite men), 0)$, $b=((bite men),0)$, where the $0$ is indicative of the absence of the topmost syntactic operation. That is, ``big dogs'' and ``bite men'' are still distinct phrases at the $L_{2}$ level, and have not resulted in a single object. The exact implementational details of how $f_{2}$ does this is underspecified at the formal level, and shown in chapter \ref{chapter:body2} at the algorithmic level, keeping in mind the possibility of multiple realisability. (How such a function can be learnt is a matter of active research, and the description of processing at the algorithmic level is the first step towards that goal). In $f_{3}:f_{3}(f_{2}(s),f_{1}(s),s)$, the combination of  ``big dogs'' and ``bite men'' occurs, and s\ref{sent_1} is mapped to (big dogs bite men, 1), where $1$ indicates the presence of the topmost syntactic operation. Similarly, for s\ref{sent_2}, the same recursive functional mapping occurs, where each image in a $P^{*}_{i}$ is the result of previous merge-operations. Therefore, the topmost syntactic operation is reached in $P^{*}_{5}$ for s\ref{sent_2}. \\

$\pi_{i}$ is a compression mechanism, which abstracts the elements of $P^{*}_{i}$ to $T^{*}_{i}$. The labelling of compositional elements uses lexical information along with time to perform indexing on the elements of $T^{*}_{i}$. This abstract representation allows multiple distinct sequences to have identical $T^{*}_{i}$ images. Take for example two sequences s\ref{sent_1} and ``Big men bite dogs''. Although the semantic information carried in $P^{*}_{i}$ is distinct - ``Big dogs'' vs ``Big men'' and ``bite men'' vs ``bite dogs'', their structural similarity is maintained because $\pi_{2}$ maps them to the same element - $(2,4,NP,VP,0)$. That is to say, the NP and VP end at positions 2,4 respectively. One can claim that this element alone is sufficient to determine grammaticality, but it is the topmost operation which binds the NP and VP is what allows the sentence to be called so. And therefore, until the topmost merge operation does not occur, processing cannot be considered complete.\\

The $i_{i}$ inclusion functions map the elements in $T^{*}_{i}$ to the Pushout generated by the equivalence relation. The pushout then is the discovery of latent structure. This is shown by the grouping of similarly coloured circles to represent equivalence classes. The universal property holds that there is a mapping $Q\to H$. Notice that there are dotted yellow and orange lines going from $T^{*}_{3}$ and $T^{*}_{5}$ to $\mathbb{H}$. Only these two elements, (which are topmost root constituents) have unique corresponding syntax trees in $\mathbb{H}$, therefore making $h_{i}$ a partial function. This is the crux of the statement - $f_{i}\circ \pi_{i}\circ h_{i} = f_{i}\circ \pi_{i}\circ i_{i} \circ (Q\to H)$.\\ 

Every sentence, when built up, requires the application of the topmost syntactic operation for membership in the syntax tree. But this is not to say that elements of other level sets do not have any significance. On the contrary, they gather evidence at every time step by claiming membership in an equivalence class in the Pushout. 

\begin{sidewaysfigure}
    \centering
    \includegraphics[scale=0.6]{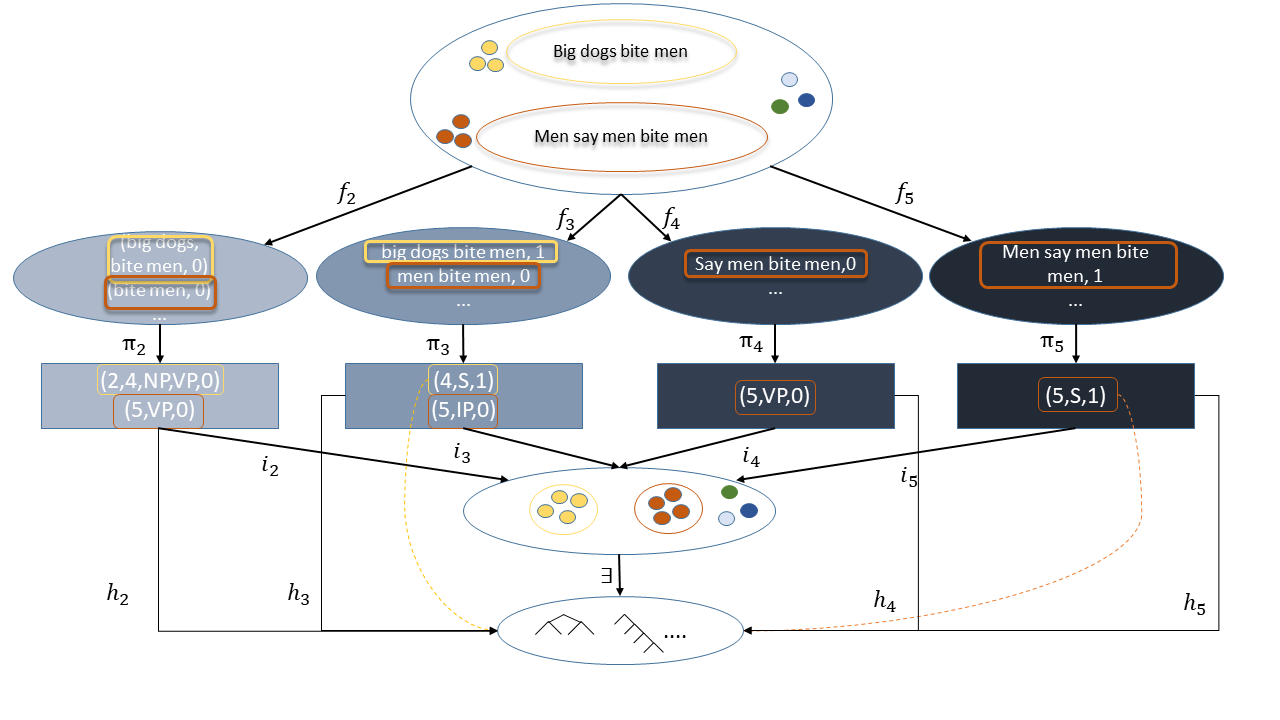}
    \caption{Detailed diagram for sample sentences}
    \label{detailed_1}
\end{sidewaysfigure}

\section{Binding by Asynchrony}\label{binding_by_asynchrony}

As shown in figure \ref{illustration_dora_binding}, DORA shows binding by asynchrony. When this asynchrony is maintained at the level of PO units, unit coding for the predicate fires right before the unit coding for the object, and together, the RB unit fires out of synchrony with other units. Activation is shown by grey units. The binding of units larger to cup is carried by the firing of unit representing larger (in \ref{fig1} i), followed by unit representing cup (in \ref{fig1} ii). To make this binding phenomenon more robust, it is necessary that a given RB unit fire more than once. That is, if (i),(ii),(iii) and (iv) are considered four time steps, then the activation profile in \ref{fig2} shows DORA repeating these four time steps twice ie., i, ii,iii,iv - i,ii,iii,iv. The yoked inhibitor serves to ensure that once the activation of the PO/RB/P unit goes over a threshold, it is brought down to zero.

\begin{figure}
\begin{center}
\subfigure[Binding by asynchrony]{
    \label{fig1}
    \includegraphics[scale=0.6]{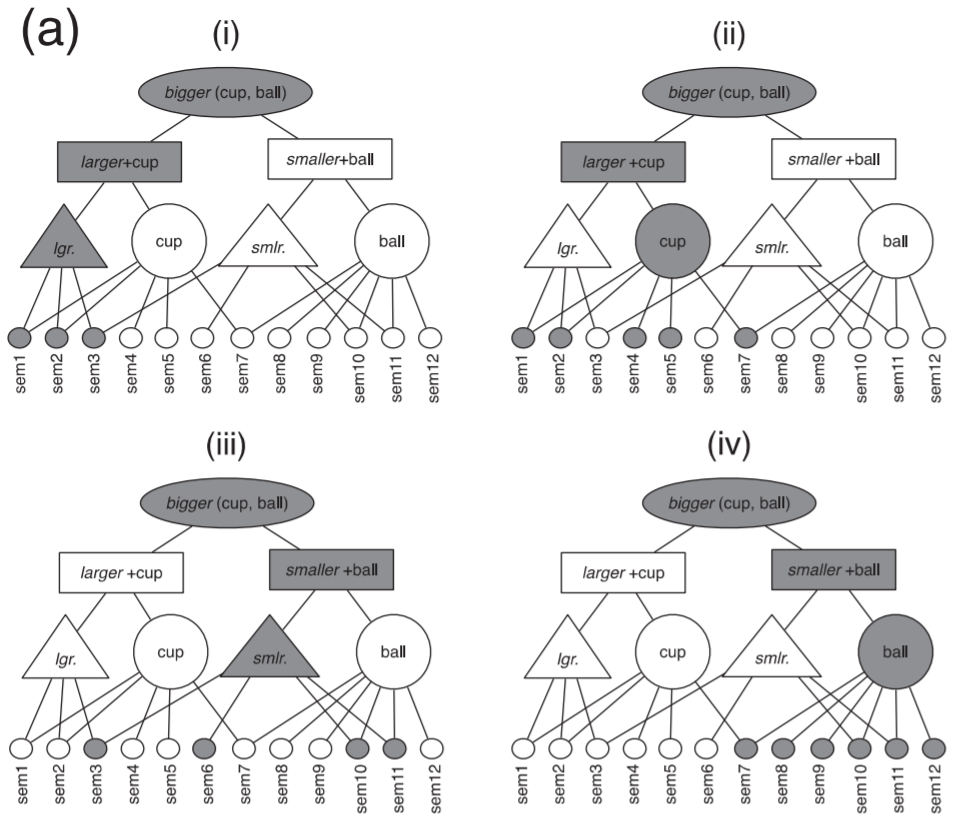}}
\subfigure[Activation profiles of units]{
    \label{fig2}
    \includegraphics[scale=0.6]{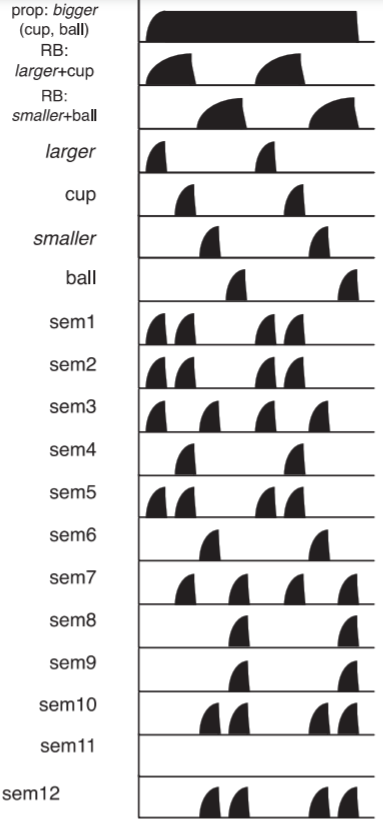}}
\end{center}
\caption{Binding by asynchrony example - from \cite{doumas_learning_2018}}
\label{illustration_dora_binding}
\end{figure}

 \clearemptydoublepage

 \printglossaries

 \addcontentsline{toc}{chapter}{Bibliography}
 \bibliography{bibliography/export-data.bib}

\begin{thebibliography}{10}
\bibitem{noauthor_gpt-3_nodate}
{GPT}-3, {Bloviator}: {OpenAI}’s language generator has no idea what it’s
  talking about.
\end{thebibliography}

\end{document}